\documentclass[sigconf]{aamas} 

\doi{JTHG8732}

\newif\ifarxiv
\arxivfalse 
\arxivtrue 

\usepackage{balance} 
\usepackage[utf8]{inputenc} 
\usepackage[T1]{fontenc}    
\usepackage{hyperref}       
\usepackage{url}            
\usepackage{booktabs}       
\usepackage{amsfonts}       
\usepackage{nicefrac}       
\usepackage{microtype}      
\usepackage{xcolor}         

\usepackage{amsmath}
\usepackage{graphicx}
\usepackage{algorithm}
\usepackage{algpseudocode}
\usepackage{multirow}
\usepackage{wrapfig}
\usepackage{subcaption}
\usepackage[capitalize]{cleveref}
\usepackage{amsthm}
\usepackage[toc,page,header]{appendix}
\usepackage{minitoc}
\usepackage{etoc}
\usepackage[most]{tcolorbox}
\makeatletter
\@namedef{ver@lineno.sty}{9999/12/31}
\@namedef{opt@lineno.sty}{}
\makeatother
\usepackage[frozencache,cachedir=minted]{minted}
\usepackage{float}

\usepackage{xspace}
\usepackage{comment}
\newcommand{\algo}{L2HR\xspace}
\newcommand{\algofullname}{Language to Hierarchical Rewards\xspace}
\newtheorem{definition}{Definition}
\newtheorem{proposition}{Proposition}
\newtheorem{property}{Property}

\newcommand\myeq{\mkern1.5mu{=}\mkern1.5mu}

\usepackage{enumitem}
\setlist[itemize]{leftmargin=*}
\setlist[enumerate]{leftmargin=*}
\usepackage[normalem]{ulem}

\usepackage{tikz}
\usetikzlibrary{shapes.geometric}
\usetikzlibrary{arrows,backgrounds,fit,positioning}

\makeatletter
\gdef\@copyrightpermission{
  \begin{minipage}{0.2\columnwidth}
   \href{https://creativecommons.org/licenses/by/4.0/}{\includegraphics[width=0.90\textwidth]{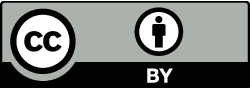}}
  \end{minipage}\hfill
  \begin{minipage}{0.8\columnwidth}
   \href{https://creativecommons.org/licenses/by/4.0/}{This work is licensed under a Creative Commons Attribution International 4.0 License.}
  \end{minipage}
  \vspace{5pt}
}
\makeatother

\setcopyright{ifaamas}
\acmConference[AAMAS '26]{Proc.\@ of the 25th International Conference
on Autonomous Agents and Multiagent Systems (AAMAS 2026)}{May 25 -- 29, 2026}
{Paphos, Cyprus}{C.~Amato, L.~Dennis, V.~Mascardi, J.~Thangarajah (eds.)}
\copyrightyear{2026}
\acmYear{2026}
\acmDOI{}
\acmPrice{}
\acmISBN{}

\acmSubmissionID{1486}

\newcommand\blfootnote[1]{%
  \begingroup
  \renewcommand\thefootnote{}\footnote{#1}%
  \addtocounter{footnote}{-1}%
  \endgroup
}

\begin{document}

\title[Hierarchical Reward Design from Language]{Hierarchical Reward Design from Language: Enhancing Alignment of Agent Behavior with Human Specifications}


\author{Zhiqin Qian}
\affiliation{
  \institution{Rice University}
  \city{Houston, TX}
  \country{USA}}
\email{bill.qian@rice.edu}

\author{Ryan Diaz}
\affiliation{
  \institution{Rice University}
  \city{Houston, TX}
  \country{USA}}
\email{ryandiaz@rice.edu}

\author{Sangwon Seo}
\affiliation{
  \institution{Rice University}
  \city{Houston, TX}
  \country{USA}}
\email{sangwon.seo@rice.edu}

\author{Vaibhav Unhelkar}
\affiliation{
  \institution{Rice University}
  \city{Houston, TX}
  \country{USA}}
\email{vaibhav.unhelkar@rice.edu}
\begin{abstract}
When training artificial intelligence (AI) to perform tasks, humans often care not only about \textit{whether} a task is completed but also \textit{how} it is performed. 
As AI agents tackle increasingly complex tasks, aligning their behavior with human-provided specifications becomes critical for responsible AI deployment.  
\textit{Reward design} provides a direct channel for such alignment by translating human expectations into reward functions that guide reinforcement learning (RL).
However, existing methods are often too limited to capture nuanced human preferences that arise in long-horizon tasks. 
Hence, we introduce \textit{Hierarchical Reward Design from Language (HRDL)}: a problem formulation that extends classical reward design to encode richer behavioral specifications for hierarchical RL agents.
We further propose \textit{\algofullname~(\algo)} as a solution to HRDL.  
Experiments show that AI agents trained with rewards designed via \algo not only complete tasks effectively but also better adhere to human specifications. 
Together, HRDL and \algo advance the research on human-aligned AI agents.
\end{abstract}




\pagestyle{fancy}


\maketitle 


\section{Introduction}
\label{sec:intro}

\begin{figure*}[t]
    \centering
    \includegraphics[height=15em]{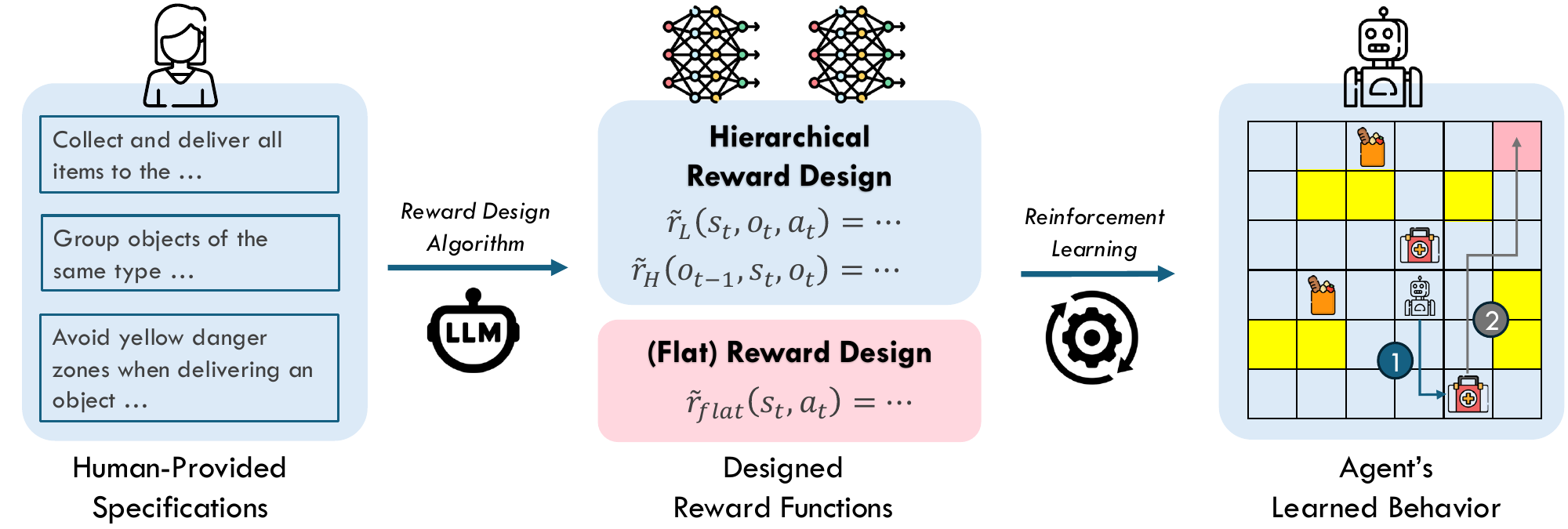}
    \caption{This work introduces the Hierarchical Reward Design from Language (HRDL) problem. Unlike prior work on reward design, HRDL decomposes reward design into low- and high-level components $(\tilde{r}_L, \tilde{r}_H)$. Language to Hierarchical Rewards (L2HR), our proposed solution to HRDL, leverages language models to guide the synthesis of these hierarchical rewards, enabling existing RL algorithms to train agents that are better aligned with human specifications.}
    \label{fig:intro}
\end{figure*}

Robots and other AI agents are increasingly being deployed in human-centric environments, such as homes, hospitals, and disaster zones~\citep{scopelliti2004if, kidd2008robots, pantofaru2012exploring, unhelkar2018human, unhelkar2019semi, gonzalez2021service, wang2024mosaic, qian2025astrid, seo2025socratic}.
Their usefulness depends not only on task completion, but also on respecting human intentions, operational rules, and safety requirements (collectively referred to as \textit{behavior specifications}).

Aligning agent behavior with these specifications is central to their responsible deployment, with various paradigms being actively explored for conveying such specifications~\citep{chernova2022robot}.
This work focuses on \textit{reward design}, a paradigm where humans convey specifications through reward functions that guide reinforcement learning (RL), typically before AI deployment.
This paradigm is particularly suitable for use cases where specifications are relatively stable and the costs of reward design can be amortized, such as during the repeated use of AI agents by humans in domain-specific contexts.
\ifarxiv
\blfootnote{This article is an extended version of an identically-titled paper accepted at the \textit{International Conference on Autonomous Agents and Multiagent Systems (AAMAS 2026)}.}
\fi

As AI agents take on increasingly complex tasks, more advanced reward design methods are needed to capture equally complex specifications.
Humans rarely reason about tasks and associated specifications as monolithic goals~\citep{miller2017plans, botvinick2009hierarchically, correa2023humans, catrambone1998subgoal, ramaraj2023analysis, kim2013learnersourcing}.
Instead, we naturally break them into subtasks.
Hierarchical frameworks in RL mirror this structure by decomposing tasks into subtasks and organizing them over long horizons~\citep{dayan1992feudal,parr1997reinforcement,sutton1999between,dietterich2000hierarchical}. 
This \textit{hierarchical approach to policy learning} has enabled agents to complete tasks of increasingly longer horizons.
However, \textit{the reward design for these hierarchical RL agents remains largely unexplored,} thereby limiting human-AI alignment.

As illustrated in \cref{fig:intro}, specifications for long-horizon tasks often include details on  \textit{what} subtasks to perform, in \textit{which} order, and \textit{how} they are executed. 
Existing reward design methods encode these specifications via a flat reward function of the form $\tilde{r}_{flat}(s, a)$.
We show \textit{theoretically and empirically that flat rewards are fundamentally limited in capturing specifications for long-horizon tasks}.
To address this limitation:
\begin{itemize}
    \item We introduce the \textit{Hierarchical Reward Design (HRD)} problem, which enables designers to express behavioral specifications inspired by the same structured way people naturally think and teach. Unlike the classical (flat) reward design problem~\citep{singh2009rewards}, HRD admits reward solutions that enable encoding of complex specifications for long-horizon tasks, capturing both what subtasks to perform and how to execute them. HRD is a general formulation that can be instantiated with multiple input modalities, analogous to how flat reward design has been realized via proxy signals or language~\citep{hadfield2017inverse, ma2023eureka, kwon2023reward}.  Because natural language is an intuitive medium for specifying layered instructions, we then provide a language-based instantiation called \textit{Hierarchical Reward Design from Language (HRDL, pronounced ``hurdle'')}.
    \item We prove that hierarchical rewards of HRDL are strictly more expressive than flat rewards used by prior works, while remaining compatible with standard decision-making frameworks (i.e. Markov and semi-Markov Decision Processes) and RL algorithms. 
    \item We present \textit{\algofullname (\algo)}, an initial solution to HRDL that generates hierarchical rewards directly from natural language specifications, making reward design more accessible while leveraging the reasoning capabilities of large language models~\citep{achiam2023gpt, huang2022towards, li2022competition}. \algo\ produces reward structures that guide both high-level subtask selection and low-level execution.
\end{itemize}

Human subject and numerical experiments demonstrate hierarchical reward design's advantages over flat reward design.
Hierarchical reward design (coupled with hierarchical RL) enables AI agents to successfully complete tasks and better align their behavior with language specifications.
We view this work as an initial but important step toward aligning AI systems with human expectations through the lens of HRD.
Through theoretical analysis and empirical findings, this paper lays the groundwork for future research on designing human-aligned reward structures that employ hierarchies and human input.

\section{Background and Related Work}
This work focuses on AI agents tasked with sequential decision-making tasks, which are commonly modeled using Markov Decision Processes (MDP) or its variants~\cite{puterman2014markov, sutton2018reinforcement}.
An MDP is defined by the tuple $\mathcal{M} \doteq (\mathcal{S}, \mathcal{A}, T$, $r, \gamma, h)$, where $\mathcal{S}$ and $\mathcal{A}$ are the state and action spaces, $T(s'|s,a)$ the transition dynamics, $r(s, a)$ the reward function, $\gamma\in[0,1]$ the discount factor, and $h$ the horizon.
MDPs can be solved using RL, which aims to find a policy $\pi(a|s)$ that maximizes the expected discounted return $\mathbb{E}[\sum_{t=0}^{h} \gamma^{t} r(s_t,a_t)]$.

\subsection{Hierarchical Reinforcement Learning}
While capable, standard RL algorithms struggle with long-horizon tasks.
Hierarchical RL (HRL) seeks to solve MDPs with long horizons by decomposing them into simpler subtasks ~\citep{dayan1992feudal,parr1997reinforcement,sutton1999between,dietterich2000hierarchical,seo2024idil,seo2025hierarchical}. 
A widely used HRL paradigm is the options framework~\cite{sutton1999between, bacon2017option, zhang2019dac}.
Here, the agent has access to a discrete set of temporally-extended behaviors called options $\mathcal{O}$. 
Each $o \in \mathcal{O}$ is associated with an intra-option policy $\pi_L(a|s,o)$ and a termination condition $\beta(o^-, s)$.
A high-level policy $\pi_H(o|o^-,s)$ selects options, while the low-level policy executes primitive actions until the selected option terminates.
The reward model for options computes the expected cumulative reward until option termination.
Following~\cite{sutton1999between}, we let $\mathcal{E}(o, s, t)$ denote the event where option $o$ is initiated in state $s$ at timestep $t$, and define the option-level reward as:%
\begin{gather} 
    r_{opt}(s, o) \doteq \mathbb{E}[\Sigma_{i=1}^k \gamma^{i-1} r_{t+i} \mid \mathcal{E}(o, s, t)]
\label{eq:taskr-until-termination}
\end{gather}%
where $k$ denotes the number of steps after which the initiated option $o$ terminates, determined by its termination condition $\beta$.

Another line of HRL research follows the \textit{feudal/goal-conditioned framework} \citep{dayan1992feudal, vezhnevets2017feudal, kulkarni2016hierarchical, nachum2018data, jiang2019language}, which also decomposes a task into subtasks but differs in how the hierarchical policies are trained and how reward signals are assigned.
In this framework, the high-level manager selects subgoals and receives a \emph{task} reward, as in the options framework.
However, unlike the options framework, the low-level worker receives a separate \emph{pseudo-reward} $r_p(s, o, a)$ that measures progress toward achieving the current subgoal.

\subsection{Reward Design}
A core challenge in using MDPs and RL is reward design~\citep{singh2009rewards, sutton2018reinforcement}.
Given a well-designed reward function, agents can use RL algorithms to solve the MDP.
However, in practice, the design of a reward function is non-trivial and can lead to a host of problems, including poor alignment between humans and agents~\cite{amodei2016concrete}.

\subsubsection{Reward Design Problem.}
To formally study reward design, \citeauthor{singh2009rewards} introduce the (flat) Reward Design Problem (RDP)~\cite{singh2009rewards}.
RDP is formalized as a tuple $P = (\mathcal{M}_p, \mathcal{R}, \mathcal{A}_{\mathcal{M}_p}, F)$, where 
\begin{itemize}
    \item $\mathcal{M}_p = (\mathcal{S}, \mathcal{A}, T, \gamma, h)$ is the \textit{world model}; 
    \item $\mathcal{R}$ is the space of reward functions; 
    \item $\mathcal{A}_{\mathcal{M}_p}(r) : \mathcal{R} \rightarrow \Pi$ is an algorithm to compute policy $\pi: \mathcal{S} \rightarrow \Delta(\mathcal{A})$ that optimizes reward $r \in \mathcal{R}$ in the MDP $(\mathcal{M}_p, r)$; and
    \item $F: \Pi \rightarrow \mathbb{R}$ is the fitness function that produces a scalar evaluation of a policy, only accessible via policy queries. 
\end{itemize}%
The RDP seeks a reward function $r \in \mathcal{R}$ such that \textit{the policy $\pi := \mathcal{A}_{\mathcal{M}_p}(r)$ that optimizes $r$} achieves the highest Fitness score $F(\pi)$.
Rather than treating the reward as fixed and exogenous, RDP reframes reward design as a \textit{search problem}, a perspective that has profoundly shaped subsequent research. 
It inspired methods that optimize or evolve reward functions directly~\cite{sorg2010reward, niekum2010genetic} as well as formulations that infer or generate them from indirect signals, such as Inverse Reward Design~\cite{hadfield2017inverse}, which recovers true rewards from proxy rewards, and Eureka~\cite{ma2023eureka}, which synthesizes executable rewards from natural language. 
Many other paradigms can be viewed as RDP instantiations: inverse RL~\cite{abbeel2004apprenticeship, ziebart2008maximum, ho2016generative, fu2017learning} treats expert behavior as evidence for the reward search, while preference-based learning~\cite{sadigh2017active, basu2018learning} and RLHF~\cite{ouyang2022training, bai2022training} extend this to human feedback.

\begin{figure}[t]
\centering
    \begin{tikzpicture}[fill=gray]
    \node (a) [%
        draw, ellipse, minimum width=4.2cm, minimum height=2.4cm] 
        {};
    \node (b) [%
        draw, ellipse, minimum width=4.2cm, minimum height=2.4cm, right of=a, xshift=1.4cm] 
        {};
    \path (a) -- (b) node[midway] (c) {\textbf{HRDL}};
    
    \node (a1) [%
        left= 5mm of c.west, yshift=4mm]
        {\it Hierarchical};
    \node (a2) [%
        left= 5mm of c.west, yshift=0mm]
        {\it Reinforcement};
    \node (a3) [%
        left= 5mm of c.west, yshift=-4mm]
        {\it Learning};

    \node (b1) [%
        right= 5mm of c.east, yshift=4mm]
        {\it Reward};
    \node (b2) [%
        right= 5mm of c.east, yshift=0mm]
        {\it Design from};
    \node (b3) [%
        right= 5mm of c.east, yshift=-4mm]
        {\it Language};
    
    \end{tikzpicture}
    
    \caption{
    Although prior works have utilized multi-level rewards to train hierarchical agents, the \textit{design} of such rewards remains underexplored and lacks a concrete problem formulation.
    This work bridges this gap by introducing HRDL.} 
    \label{fig:related-work}
\end{figure}
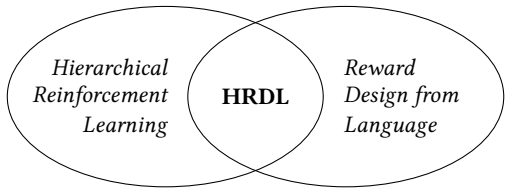%

\subsubsection{Rewards Design for Hierarchical RL.}
While RDP provides a unifying framework that has catalyzed advances in both RL and human-AI alignment, \textit{it does not explicitly consider hierarchical RL}.
Many prior works have explored the use of hierarchical rewards, from early studies in feudal reinforcement learning~\cite{dayan1992feudal} and precursors to the options framework~\cite{singh1992transfer} to more recent advances in deep HRL~\cite{vezhnevets2017feudal,kulkarni2016hierarchical, nachum2018data}.
While these works often \textit{assume access} to hierarchical rewards for training agents to complete tasks, \textbf{the problem of \textit{designing} such hierarchical rewards has received little attention and, to our knowledge, has yet to be formally defined as a concrete research problem}.

A literature search using the keywords ``hierarchical reward design'' primarily returns domain-specific studies that discuss how \textit{using} hierarchical rewards enables solving application-level problems~\cite{huang2023multi, newton2022hierarchical, chen2023hierarchical, naito2025task}.
Other works either employ the term \textit{hierarchy} in different contexts~\cite{sun2019intelligent, li2024bounded, lai2024alarm, jin2024reward}, such as to express the relative importance of multiple flat reward signals~\cite{li2024bounded}.
A related but conceptually distinct line of work utilizes Reward Machines (RMs)~\cite{icarte2022reward, furelos2023hierarchies}, which represent rewards using automata or temporal logic.
RM-based approaches focus on exploiting known reward structures to improve \textit{learning efficiency}, rather than on studying how such structures can be specified from human input.

Since no previous work formally defines the \textit{design of hierarchical rewards} as a general problem, there is a lack of consistent language and theoretical foundation for studying it.
This work addresses this gap at the intersection of hierarchical rewards and reward design (Figure~\ref{fig:related-work}) by formalizing the Hierarchical Reward Design problem and introducing a reward structure that is hierarchical, compact, and capable of capturing nuanced behavioral specifications for both \textit{what} subtasks to select and \textit{how} to execute them.

\textit{Just as RDP lays the groundwork for studying algorithmic reward design in flat settings, we posit that HRD will provide a principled foundation for reasoning about hierarchical reward structures.} In line with the research that originated from RDP, we anticipate that HRD will enable a broad range of problem instantiations (of reward design with different types of human inputs) and solution methods for designing hierarchical rewards.

\subsubsection{Reward Design from Language.}
Early work addressing the RDP focused on designing rewards for intrinsic motivation and reward shaping~\citep{singh2010intrinsically, sorg2010internal, sorg2010reward}.
More recently, research in this area has explored aligning agent behavior more closely with human-provided specifications, using learning or large language models (LLMs) to infer and generate reward functions~\citep{hadfield2017inverse, kwon2023reward}.
In these cases, the human or an oracle, either implicitly or explicitly, serves as the fitness function by evaluating the policy.
However, most existing work on reward design or generation focuses exclusively on non-hierarchical (flat) RL settings, producing reward functions of the form $r_{flat}(s, a)$ or $r_{flat}(s)$ \cite{singh2010intrinsically, sorg2010internal, sorg2010reward, hadfield2017inverse, kwon2023reward, du2023guiding, han2024autoreward, ma2023eureka, li2024auto, xie2023text2reward, bhambri2024extracting, yu2023language}.
While sufficient for certain behaviors, \textit{flat reward functions are fundamentally limited when specifying complex preferences, such as desired subtask sequences or option-conditioned execution strategies, that naturally arise in long-horizon tasks}.

To our knowledge, the only prior work that explicitly considers a hierarchical setting is~\cite{masadome2025reward}, though its focus differs substantially from ours.
Their approach does not formalize the hierarchical reward design problem or analyze the expressivity gap between flat and hierarchical formulations.
Moreover, the rewards generated by their LLMs are limited to task completion objectives and do not capture behavioral specifications.
In contrast, \algo generates both high- and low-level rewards that encode natural language behavioral preferences while preserving task feasibility.%
\footnote{Please see the Appendix for additional details on related works.}

\section{Hierarchical Reward Design}
\label{sec:hrd}

This section formally introduces the Hierarchical Reward Design problem.
We begin by defining the low- and high-level reward functions in HRD and proceed to show that they naturally induce a family of MDPs at the low-level and a semi-MDP (SMDP) at the high-level.
Using these insights, we formally define the general HRD problem along with a specific instantiation, the \textit{Hierarchical Reward Design from Language (HRDL)} problem.\footnote{%
Proofs for all propositions are also provided in the Appendix.}


\subsection{Low‑level and High-level Reward Models}

\begin{definition} [Low-level Reward]
The \textbf{low-level reward} is a function
$r_L: \mathcal{S} \times \mathcal{O} \times \mathcal{A} \rightarrow \mathbb{R}$, which provides feedback for selecting a low-level action $a \in \mathcal{A}$ in state $s \in \mathcal{S}$ while pursuing option $o \in \mathcal{O}$.\label{def: low-level reward}
\end{definition}

Intuitively, $r_L(s, o, a)$ encodes specifications for \textit{how} the agent should execute the subtask associated with option $o$ in state $s$.

\begin{proposition} [Low-level MDP Models]
Let $\mathcal{M}_p = (\mathcal{S}, \mathcal{A}, T, \gamma)$ be a world model, $\mathcal{O}$ a set of options, and $r_L: \mathcal{S} \times \mathcal{O} \times \mathcal{A} \rightarrow \mathbb{R}$ a low‑level reward. 
For a fixed option $o \in \mathcal{O}$, the tuple $\mathcal{M}_{L, o} = (\mathcal{S}, \mathcal{A}, T, r_L(\cdot, o, \cdot), \gamma, h_o)$ defines an MDP, where $h_o$ is the horizon determined by the option’s termination condition $\beta(\cdot, o)$.
\label{prop:low-level}
\end{proposition}

\begin{definition} [High-level Reward]
The \textbf{high-level reward} is a function
$r_H: \mathcal{O} \times \mathcal{S} \times \mathcal{O} \rightarrow \mathbb{R}$, which specifies the expected reward for executing option $o \in \mathcal{O}$ until termination, given that $o$ is initiated in state $s \in \mathcal{S}$ and the previous option was $o^- \in \mathcal{O}$.
\label{def: high-level reward}
\end{definition}

The high-level reward $r_H(o^-, s, o)$ encodes specifications for \textit{what} subtask should be executed, possibly conditioned on both the current state and prior option.
This allows for expressing preferences over subtask \textit{ordering} and dependencies between subtasks.

\begin{proposition}[High-level SMDP Model] \leavevmode\par
\noindent Let $\mathcal{M}_p= (\mathcal{S}, \mathcal{A}, T , \gamma, h)$ be a world model, $\mathcal{O}$ a set of options, and $r_H : \mathcal{O} \times \mathcal{S} \times \mathcal{O} \rightarrow \mathbb{R}$ the high‑level reward. Then, $\mathcal{M}_H = (\mathcal{O} \times \mathcal{S}, \mathcal{O}, T_H, r_H, \gamma, h)$ forms a semi-MDP, where $T_H: \mathcal{O} \times \mathcal{S} \times \mathcal{O} \rightarrow \Delta(\mathcal{O} \times \mathcal{S} \times \mathbb{N})$ defines the joint distribution over the next augmented state and transit time, where $\mathbb{N}$ is the set of natural numbers.
\label{prop:high-level-smdp}
\end{proposition}

Alternatively, the high-level process can be modeled as a standard MDP when \textit{single-step} high-level rewards are used.
This flexibility highlights that the HRD framework is compatible with both semi-MDP and MDP formulations, allowing the use of a wide range of existing RL algorithms.
We provide the formal MDP definition and corresponding proofs in the Appendix.

\subsection{The HRD Problem}
\label{sec:hrd-problem}
\begin{definition} [Hierarchical Reward Design (HRD)]
The \textbf{Hierarchical Reward Design (HRD)} problem is defined by the tuple $P = (\mathcal{M}_p, \mathcal{O}, \mathcal{R}, \mathcal{A}_{\mathcal{M}_p}, F)$, where 
\begin{itemize}
    \item $\mathcal{M}_p = (\mathcal{S}, \mathcal{A}, T, \gamma, h)$ is the world model;
    \item $\mathcal{O}$ is a finite option set; 
    \item $\mathcal{R} = \mathcal{R}_H \times \mathcal{R}_L$ is the space of candidate reward structures, where $\mathcal{R}_H = \{r_H: \mathcal{O} \times \mathcal{S} \times \mathcal{O} \rightarrow \mathbb{R}\}$ and $\mathcal{R}_L = \{r_L: \mathcal{S} \times \mathcal{O} \times \mathcal{A} \rightarrow \mathbb{R}\}$;
    \item the learning routine $\mathcal{A}_{\mathcal{M}_p}(\cdot) : \mathcal{R} \rightarrow \Pi_H \times \Pi_L$ maps each reward pair $(r_H, r_L)$ to a hierarchical policy $(\pi_H, \pi_L)$, where $\pi_H: \mathcal{O} \times \mathcal{S} \rightarrow \Delta(\mathcal{O})$ optimizes $r_H$ in the high-level decision making model $\mathcal{M}_H$ and $\pi_L: \mathcal{S} \times \mathcal{O} \rightarrow \Delta(\mathcal{A})$ optimizes $r_L$ in each underlying MDP $\mathcal{M}_{L, o}$; and
    \item the fitness function $F: \Pi_H \times \Pi_L \rightarrow \mathbb{R}$ evaluates the quality of hierarchical policies. 
\end{itemize}
The goal of HRD is to find $(r_H^*, r_L^*) = \arg \max_{(r_H, r_L) \in \mathcal{R}} F(\mathcal{A}_{\mathcal{M}_p}(r_H, r_L))$.
\end{definition}

\subsubsection*{Connections to Existing Algorithms.} 
We show in the Appendix that $\mathcal{A}_{\mathcal{M}_p}$ can be instantiated with existing RL algorithms.
In our implementation, the low-level policy $\pi_L(a\mid s,o)$ is trained with PPO~\cite{schulman2017proximal} due to its robustness in control, while the high-level policy $\pi_H(o\mid o^-, s)$ uses DQN-style methods~\cite{mnih2015human}, following common practice in SMDP formulations ~\cite{sutton1999between, bacon2017option}. 
Stronger structural assumptions on $(r_H, r_L)$ can enable the use of more specialized HRL algorithms.
For instance, if the low-level reward depends only on state and action, $r_L(s, o, a)=r_{flat}(s, a)$, and the high-level reward $r_H$ is a single-step reward constructed as $\sum_{a} r_{flat}(s, a) \pi_L(a|s, o)$, the problem reduces to the two augmented MDPs formulation introduced in \cite{zhang2019dac}.
In these cases, algorithms such as double actor-critic \cite{zhang2019dac} and option-critic \cite{bacon2017option} can be applied to learn hierarchical policies.
A detailed exploration of the connections between structural reward assumptions and the applicability of existing HRL algorithms for instantiating $\mathcal{A}_{\mathcal{M}_p}$ is left for future work.

\subsection{Hierarchical Reward Design from Language}
\label{sec:hrdl-problem}
As discussed in Sec.~\ref{sec:intro}, real-world deployments often require agents to satisfy additional behavioral specifications beyond task completion.
In these cases, the task reward can typically be defined once and reused across different behavioral contexts.
In contrast, additional rewards must be redesigned for each new behavior specification.
While the cost of task reward design is amortized, the cost of designing rewards that match human specifications grows linearly with the number of distinct behaviors desired.
This motivates the need for an automated approach to generate rewards to encode behavioral specifications while reusing the existing domain dynamics and task objectives.
The challenge of this problem is twofold:
(1) The generated rewards should have distinct functional forms -- one guiding high-level option selection, and another governing low-level action execution.
(2) The generated rewards must remain compatible with existing task rewards, ensuring that agents continue to achieve the original task objectives.
We formally define this as a specific instantiation of the HRD problem.

\begin{definition}[Hierarchical Reward Design from Language (HRDL)]
The \textbf{HRDL} problem is an instance of the HRD problem with additional inputs: (1) a task reward function $r: \mathcal{S} \times \mathcal{A} \rightarrow \mathbb{R}$, (2) a subtask completion reward (pseudo-reward) $r_p: \mathcal{S} \times \mathcal{O} \times \mathcal{A}\rightarrow \mathbb{R}$, and (3) behavior specifications $l \in \Sigma^*$, provided as a natural language description.
$l$ guides the reward generation during training, and the fitness function $F$ is accessible \textbf{only} during evaluation. 
The objective of HRDL is to generate high- and low-level designed rewards, $R^*=(\tilde{r}_H^*, \tilde{r}_L^*) \in \mathcal{R}$, such that the resulting hierarchical policy $(\pi_H^*, \pi_L^*)$, trained under the composite rewards $(r_{opt}+\tilde{r}_H^*, r_p+\tilde{r}_L^*)$ using $\mathcal{A}_{\mathcal{M}_p}$, maximizes the fitness score:
$(\tilde{r}_H^*, \tilde{r}_L^*) = \arg \max_{(\tilde{r}_H, \tilde{r}_L) \in \mathcal{R}} F(\mathcal{A}_{\mathcal{M}_p}(r_{opt}+\tilde{r}_H, r_p+\tilde{r}_L))$.
\end{definition}

If a \textit{non-hierarchical} reward design method is used, the designed reward has the flat form $\tilde{r}_{flat}(s, a)$.
To integrate this flat reward into the hierarchical setting, we must decompose it into high- and low-level rewards:
\begin{align}
    r_L(s, o, a) &= r_p(s, o, a) + \tilde{r}_{flat}(s, a) \\
    r_H(s, o) &= r_{opt}(s, o) + \tilde{r}_{flat, H}(s, o)
\end{align}
where $\tilde{r}_{flat, H}(s, o)$ aggregates $\tilde{r}_{flat}(s, a)$ using the same expression as Eq.~\ref{eq:taskr-until-termination}.
While flat designed rewards $\tilde{r}_{flat}(s, a)$ can encode some behavior specifications, the definitions of high- and low-level rewards in HRD provide a significantly more expressive mechanism for specifying agent behavior:
\begin{align}
    r_L(s, o, a) &= r_p(s, o, a) + \tilde{r}_L(s, o, a) \label{eq:flat-low} \\
    r_H(o^-, s, o) &= r_{opt}(s, o) + \tilde{r}_H(o^-, s, o) \label{eq:flat-high}
\end{align}
In fact, the flat reward is a special case of hierarchically designed rewards, where $\tilde{r}_L(s, o, a) = \tilde{r}_{flat}(s, a)$ and $\tilde{r}_H(o^-, s, o) = \tilde{r}_{flat, H}(s, o)$. 
The hierarchical formulation is strictly more general than the flat formulation, offering greater expressiveness in the following ways: 

\begin{property}
    Certain specifications on \textit{sub-task selection} can be expressed through $\tilde{r}_H(s, o^-, o)$, but they cannot be expressed by flat function: $\tilde{r}_{flat}(s, a)$. 
\end{property}

\begin{property}
    Certain specifications on \textit{sub-task execution} can be expressed through $\tilde{r}_L(s, o, a)$, but they cannot be expressed by flat function: $\tilde{r}_{flat}(s, a)$. 
\end{property}

Proofs of these properties are provided in the Appendix.
In the following sections, we introduce an algorithm for generating hierarchical rewards $(\tilde{r}_H, \tilde{r}_L)$ from natural language specifications and empirically compare its performance against flat reward design that generates alignment rewards of the form $\tilde{r}_{flat}(s, a)$.

\section{\algofullname~(\algo)}
\label{sec:method}
We now present \algo, an algorithm designed to solve HRDL by leveraging large language models (LLMs).
L2HR represents reward functions as programs, which it generates directly from natural language specifications.
Specifically, it first employs a structured prompting strategy to capture specifications as  prompts, which are then used as inputs to a training module that produces hierarchical reward functions.
An illustration of \algo's input and output is provided in Fig.~\ref{fig:l2hr-example}, with more details available in the Appendix.

\subsection{Prompting Strategy}
To generate executable reward functions from natural language, \algo requires the specifications to be provided as a structured prompt.
We design a prompting strategy for \algo inspired by Eureka, a recent approach for LLM-based reward generation~\citep{ma2023eureka}.
However, unlike Eureka, \algo's prompt design focuses on producing feasible reward functions without relying on access to a fitness $F$ during reward code generation.
Specifically, within \algo, the LLM is provided with a structured prompt comprising:
\begin{enumerate}
    \item \textit{Task Description}: 
    A natural language description of the overall task objective, including the approximate scale of the task reward. 
    Since RL is sensitive to reward magnitudes, this information helps ensure that the generated preference rewards are neither too small to influence learning nor so large as to destabilize training.
    The task reward code itself is intentionally withheld to reflect realistic settings in which only the reward signal (but not the code) is accessible and to prevent overfitting to implementation details.
    \item \textit{Environment Code Context}: 
    Snippets of environment source code that expose the state and action spaces without revealing simulation internals, following the methodology of~\citep{ma2023eureka}. 
    \item \textit{Relevant Action-Related Spaces}: Descriptions of the option space $\mathcal{O}$ and action space $\mathcal{A}$, including the semantic role of each. We include these descriptions to help the LLM correctly distinguish between high-level and low-level decision spaces. 
    \item \textit{Behavior Specification}:  A natural language string that describes the desired agent behaviors beyond task completion. 
    \item \textit{Formatting and Reward Design Tips}: 
    Coding constraints (e.g., avoiding global variables) and guidance on balancing designed rewards with the underlying task reward.
\end{enumerate}

\subsection{Training Procedure}
Given the structured prompt, state-of-the-art LLMs can generate plausible reward code in a zero-shot manner.
However, we find that zero-shot generation often yields reward code with syntax errors and invalid variable references in practice.
To mitigate these issues, \algo incorporates a training procedure for reward code generation that couples LLMs with RL.
Specifically, \algo generates $k$ reward candidates independently from the LLM and apply a lightweight filtering process to ensure validity.
During filtering, \algo verifies whether the code compiles without syntax errors, and whether it references only permitted state, option, and action variables exposed in the environment prompt.
We find that at least one sample in the batch passes these checks.
As a result, we forego more complex iterative refinement strategies, such as evolutionary search or reward reflection \citep{ma2023eureka}, leaving their integration into \algo as a direction for future work.

\begin{figure}[t]
    \centering
    \begin{subfigure}[b]{\columnwidth}
        \begin{tcolorbox}[label={box:l2hr_prompt}]
\footnotesize
\textbf{Task description:} The objective is to pick up all apples and eggs on the dining table and place them in the sink... \\
\textbf{Environment context:}

\begin{minted}[fontsize=\scriptsize]{python}
'''
Background: PnP_LL_Actions = [...], PnP_HL_Actions = [...] ...
'''
class ThorPickPlaceEnv(gym.Env):
    def __init__(...): ...
\end{minted}
\smallskip
\textbf{Relevant task spaces:} The agent's option/subtask space consists of picking up and placing the two types of objects... \\
\textbf{Low-level user preference:} The agent should avoid the stool while on its way to pick up or drop an egg... \\
\textbf{High-level user preference:} The agent should pick up an object type that is different from what was previously picked... \\
\textbf{Additional info:} Do not use function attributes or global variables...
        \end{tcolorbox}
        \caption{Natural Language Specifications}
        \label{fig:prompt}
    \end{subfigure} \\
    \vspace{1em}
    \begin{subfigure}[b]{\columnwidth}
        \centering
        \begin{tcolorbox}[label={box:l2hr_output}]
\footnotesize
\begin{minted}[fontsize=\scriptsize]{python}
# High-level preference reward
def get_high_level_pref_gpt(state: Dict, prev_option: int, option: int) -> 
    Tuple[float, Dict[str, float]]:
    ...
    return reward, reward_components

# Low-level preference reward
def get_low_level_pref_gpt(state: Dict, option: int, action: int) -> 
    Tuple[float, Dict[str, float]]:
    ...
    return reward, reward_components
\end{minted}
        
        \end{tcolorbox}
        \caption{Reward Functions designed by L2HR}
        \label{fig:reward}
    \end{subfigure}
    \caption{
    Illustration of L2HR input and output.}
    \label{fig:l2hr-example}
\end{figure}

The full two-stage training procedure of \textbf{\algo} is provided in the Appendix.
In the first stage, \algo uses the LLM to generate $k$ candidate low-level alignment functions $\tilde{r}_L^{(1)}, \dots, \tilde{r}_L^{(k)}$ from the specification $l$. 
The low-level prompt differs from the high-level one only in its description of the action space and level-specific behavior specifications. 
Each $\tilde{r}_L^{(i)}$ is then used to train a corresponding low-level policy $\pi_L^{(i)}$ with the combined objective of pseudo-rewards plus the LLM-generated $\tilde{r}_L^{(i)}$. 
Only the policies $\pi_L^{(i)}$ that surpass a predefined threshold of subgoal completion, based on cumulative pseudo-rewards, are considered as \textit{valid}.

In the second stage, \algo generates $k$ candidate high-level alignment functions $\tilde{r}_H^{(1)}, \dots, \tilde{r}_H^{(k)}$ and uses them to train high-level policies $\pi_H^{(j)}$, each conditioned on a \textit{valid} $\pi_L^{(i)}$ from the first stage. 
These policies are optimized using with the combined objective of option-level task reward plus the LLM-generated $\tilde{r}_H^{(j)}$.
Then, \algo returns all designed rewards $(\tilde{r}_H^{(j)}, \tilde{r}_L^{(i)})$ and corresponding trained policy pairs $(\pi_H^{(j)}, \pi_L^{(i)})$ that achieve cumulative task rewards above a threshold.
This two-stage structure promotes modularity and allows for selective reuse of subtask policies.
\section{Experiments}
\label{sec:experiments}

We empirically evaluate whether framing reward generation as a hierarchical problem offers advantages over the traditional flat (non-hierarchical) formulation. 
Specifically, we aim to evaluate:
\begin{enumerate}[label=\textbf{Q\arabic*}.]
    \item Are the generated alignment rewards syntactically correct?
    \item Do they preserve task feasibility?
    \item Do they successfully induce behaviors that match the provided specifications?
\end{enumerate}
To investigate \textbf{Q3}, we additionally conduct user studies.
The Appendix includes further experimental details, including 
more information about the domains, metrics, LLM prompts, hyperparameters, and user study protocol.

\subsection{Baselines}
To evaluate L2HR (referred to as \textit{Hier} in the tables), we consider the following baselines. 
For all LLM-based experiments, we generate $k = 8$ reward function candidates per trial using GPT-4o \cite{hurst2024gpt} and repeat this process 3 times, resulting in a total of 24 reward candidates per configuration.

\textbf{Language to Flat Reward (Flat).}
This baseline generates flat alignment reward $\tilde{r}_{flat}(s, a)$ from language and incorporates it into the hierarchical setting using Eq.~\ref{eq:flat-low} and ~\ref{eq:flat-high} to maintain a consistent two-level training framework. 
Prompts are identical to those used for \algo, except that the prompts for flat reward generation exclude option-related instructions (since flat rewards are independent of options) and express user preferences as a single combined description rather than separate high- and low-level specifications.

\textbf{No Reward Design (Task).} 
This baseline does not utilize LLMs and uses only the task reward $r$ and pseudo-reward $r_p$ without incorporating any behavioral specifications (i.e., $\tilde{r}_H = \tilde{r}_L = 0$).
As this setting is less noisy, we run it for 5 trials.

\subsection{Domains}
We evaluate the approaches on three long-horizon, multi-subtask domains: Rescue World, iTHOR, and Kitchen.
\begin{figure}[t]
    \centering
    \begin{subfigure}[b]{0.49\columnwidth}
        \centering
        \includegraphics[width=\textwidth]{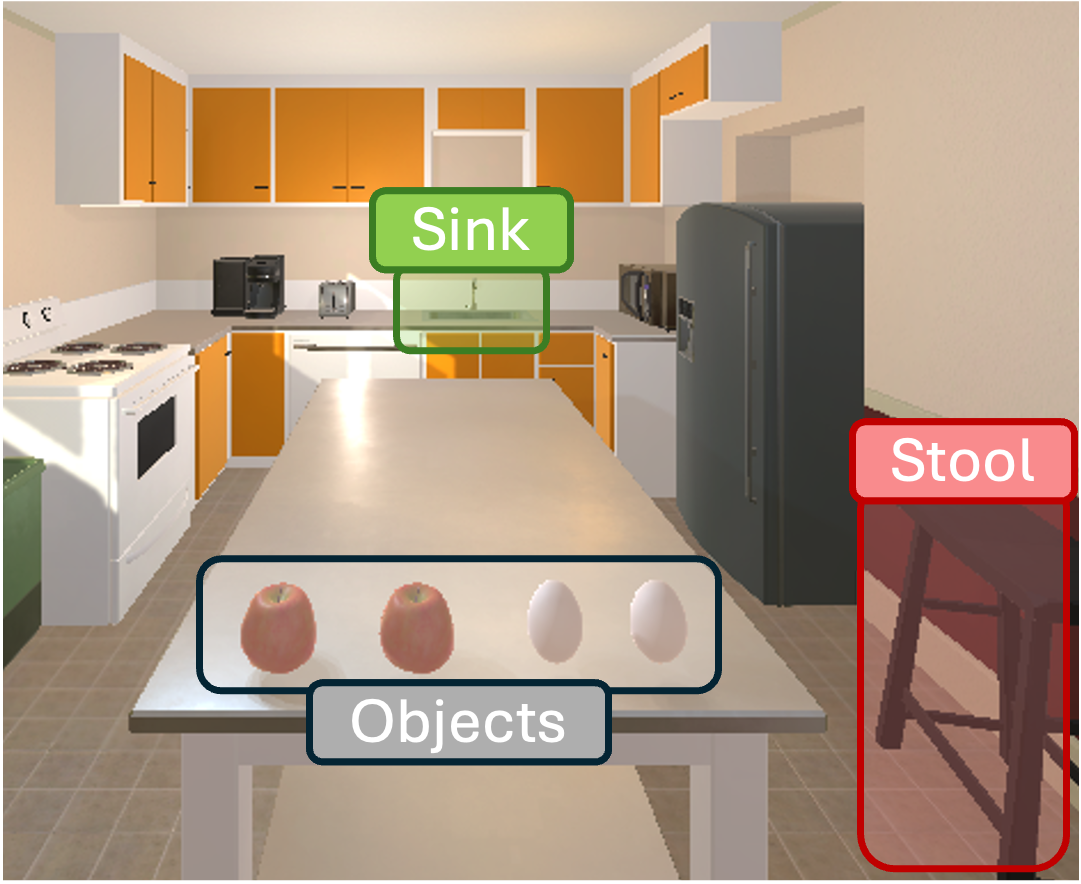} 
        \caption{iTHOR Domain}
        \label{fig:ithor}
    \end{subfigure}
    \hfill
    \begin{subfigure}[b]{0.49\columnwidth}
        \centering
        \includegraphics[width=0.985\textwidth]{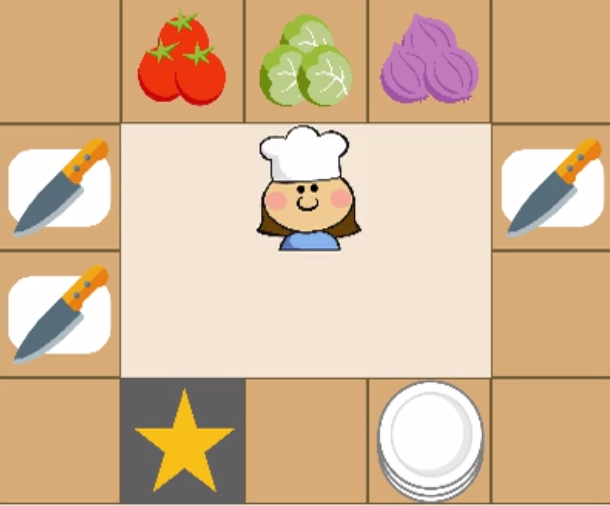}
        \caption{Kitchen Domain}
        \label{fig:kitchen}
    \end{subfigure}
    \caption{Rendered scenes from the iTHOR and Kitchen domains. Fig.~\ref{fig:intro} (right) illustrates the Rescue World domain.}
    \label{fig:domain-viz}
\end{figure}
Each domain poses distinct challenges in subtask sequencing and low-level execution.
Furthermore, as we see later, flat rewards cannot in principle capture all specifications in Rescue World and iTHOR, whereas they can in Kitchen.
This contrast allows us to investigate, \textit{in practice}, how effectively language specifications can be translated into flat and hierarchical rewards across both types of scenarios.
Additional details are provided in the Appendix and Fig.~\ref{fig:domain-viz}.

\textbf{Rescue World.} 
This is a variant of the Rescue World for Teams (RW4T) domain~\cite{orlov2024rw4t}, where the agent must collect and deliver all supplies in the environment.
This domain features a large state space represented by a 407-dimensional vector and poses a long-horizon challenge, requiring the agent to complete 8 subtasks, each lasting up to around 10 steps.
Behavioral specifications include:
(1) a high-level \textit{persistence} specification for delivering all supplies of one type before switching to another and (2) a low-level \textit{safety} specification for avoiding hazardous zones while carrying supplies.

\textbf{iTHOR.}
Built upon the Unity game engine, iTHOR is an environment within the AI2-THOR \cite{kolve2017ai2} framework that features several realistic household scenes in which an agent navigates and interact with everyday objects. 
Here, we focus on a long-horizon pick and place task within a kitchen setting consisting of 8 subtasks, each requiring up to approximately 30 steps to complete. 
The agent must deliver a set of apples and eggs located on the dining table to the sink on the other side of the room. 
The state space is represented by a 30-dimensional vector that contains object and agent positions and object states. 
Behavioral specifications include: (1) a high-level \textit{diversity} specification that requires delivering a different item from the one previously delivered and (2) a low-level \textit{avoidance} specification that prevents the agent from going near a stool placed in the environment while picking up or delivering an egg.

\textbf{Kitchen.} 
This is a single-agent variant of Overcooked, an environment originally developed for studying human-AI collaboration in kitchen tasks~\cite{liu2023llm, wang2025zsc}.
In our setting, the agent needs to prepare a salad with lettuce, tomatoes, and onions.
This domain features an even larger state space, represented by a 699-dimensional continuous vector that captures various ingredient states.
It also involves a long-horizon task requiring the completion of 5 subtasks in a strict sequence, with the final 2 subtasks dependent on the successful completion of all preceding ones.
The high-level behavioral specification is a preferred \textit{chopping} sequence (e.g., tomatoes $\rightarrow$ onions $\rightarrow$ lettuce).
Since the environment uses fixed low-level policies, we skip low-level reward design in this domain.

\subsection{Numerical Experiments}
\label{sec:results}
As a precursor to evaluting solutions to HRDL, we also conducted experiments where hand-crafted flat and hierarchical rewards were used directly to train policies \textit{without requiring reward generation from language specifications}. 
These experiments serve as a proof-of-concept and demonstrate that: 
\begin{itemize}
    \item given expert-specified hierarchical rewards $(\tilde{r}_H^*, \tilde{r}_L^*)$, existing RL algorithms can effectively learn hierarchical policies $(\pi_H^*, \pi_L^*)$ that achieve high task performance and strong alignment with designer specifications; and 
    \item while expert-specified flat rewards $\tilde{r}_{flat}^*$ can capture some behavioral specifications, they fail to express ones that require knowledge of the previous subtask (e.g., the \textit{persistence} specification in Rescue World and the \textit{diversity} specification in iTHOR).
\end{itemize}
We report these preliminary results in the Appendix.
Now, we return to the core HRDL setting, where designed rewards must be synthesized from natural language inputs.

\begin{table*}[t]
\centering
\caption{
Table showing the performance of policies trained with the task reward alone or combined with LLM-generated flat or hierarchical rewards.
For each metric, we report both the cumulative reward returns and the percentage of policies at expert-level alignment (attaining the maximum possible cumulative return for that metric).
Means and standard deviations are computed over all runs for the \textit{Task} baseline, and only over the LLM-generated reward candidates that successfully complete the task for \textit{Flat} and \textit{Hier}.
}
\label{tab:sim-results}
\begin{center} 
\resizebox{0.8\textwidth}{!}{
        \begin{tabular}{llcccccc}
            \toprule
            \multirow{2}{*}{\textbf{Domain}} & 
            \multirow{2}{*}{\textbf{Method}} & 
            \multicolumn{2}{c}{\textbf{High-Level}} & 
            \multicolumn{2}{c}{\textbf{Low-Level}} & 
            \multicolumn{2}{c}{\textbf{Total}} \\
            \cmidrule(lr){3-4}\cmidrule(lr){5-6}\cmidrule(lr){7-8} &  & \textbf{Rewards $\uparrow$} & \textbf{ Expert \% $\uparrow$} & \textbf{Rewards $\uparrow$} & \textbf{Expert \% $\uparrow$} & \textbf{Rewards $\uparrow$} & \textbf{Expert \% $\uparrow$}
            \\ \midrule
            \multirow{3}{*}{Rescue} & Task & 11.22 $\pm$ 5.57 & 20.00 & -16.46 $\pm$ 5.49 & 0.00 & 73.80 $\pm$ 5.70 & 0.00 \\
            & Flat & 9.38 $\pm$ 7.02 & 12.50 & -2.62 $\pm$ 5.19 & 62.50 & 85.13 $\pm$ 9.33 & 12.50 \\
            & Hier & \textbf{16.65 $\pm$ 6.93} & \textbf{76.92} & \textbf{-0.69 $\pm$ 1.58} & \textbf{76.92} & \textbf{93.98 $\pm$ 9.01} & \textbf{69.23} \\
            \midrule
            \multirow{3}{*}{iTHOR} & Task & 4.10 $\pm$ 1.34 & 0.00 & -23.38 $\pm$ 1.86 & 0.00 & 12.31 $\pm$ 0.66 & 0.00 \\
            & Flat & 7.67 $\pm$ 4.48 & 0.00 & -35.20 $\pm$ 12.42 & 0.00 & 3.27 $\pm$ 9.31 & 0.00 \\
            & Hier & \textbf{14.19 $\pm$ 2.23} & \textbf{87.50} & \textbf{-3.75 $\pm$ 8.14} & \textbf{75.00} & \textbf{37.68 $\pm$ 6.68} & \textbf{62.50} \\
            \midrule
            \multirow{3}{*}{Kitchen} & Task & 0.00 $\pm$ 0.00 & 0.00 & -- & -- & 0.75 $\pm$ 0.00 & 0.00 \\
            & Flat & 0.06 $\pm$ 0.13 & 10.00 & -- & -- & 0.80 $\pm$ 0.12 & 10.00 \\
            & Hier & \textbf{0.39 $\pm$ 0.05} & \textbf{92.86} & -- & -- & \textbf{1.08 $\pm$ 0.05} & \textbf{92.86} \\
            \bottomrule
        \end{tabular} 
        } 
    \end{center}
\end{table*}

\textbf{Q1}. \textbf{Are the designed rewards syntactically correct?}
Figure~\ref{fig:syntax-error} shows that hierarchical reward generation achieves substantially lower code generation error rates than flat reward generation in all three domains, \textit{suggesting that formulating the HRL reward design problem hierarchically can simplify reward synthesis for LLMs}. 
In Rescue World and iTHOR, flat rewards cannot capture the high-level specifications, as doing so requires access to the previously executed option. 
As a result, the LLM often hallucinates unavailable variables related to previously delivered object types.
In Kitchen, higher error rates stem from the complexity of reasoning over low-level actions. 
For example, correctly checking if the agent is chopping an onion on the low-level requires inspecting coordinate-level state variables.
In contrast, having access to the options space in HRD enables direct reasoning over high-level behaviors (e.g., \texttt{chop onion}), greatly simplifying reward generation.

\begin{figure}
  \centering
  \includegraphics[width=0.73\linewidth]{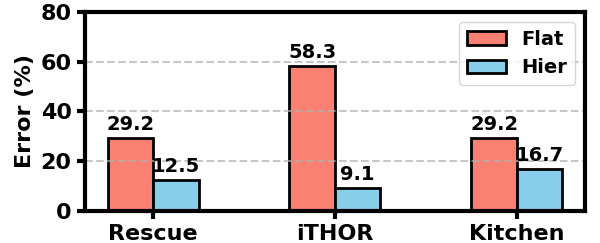}
  \caption{Syntax error rates for LLM-based reward generation, computed over 24 candidates per configuration.}
  \label{fig:syntax-error}
\end{figure}

\textbf{Q2}. \textbf{Do the designed rewards preserve task feasibility?}
\textit{Hier} consistently better preserves task feasibility than \textit{Flat} across all domains (Figure~\ref{fig:task-comp}), which is essential for real-world deployment.
In Rescue World and iTHOR, flat rewards often rely on spurious heuristics to infer past behavior (e.g. inferring the last delivered item type from the agent's current location), leading to unintended behaviors. 
In iTHOR, flat rewards tend to misidentify the type of objects being picked up, resulting in incorrect reward accumulation.
In Kitchen, flat rewards frequently struggle to reason over low-level actions (e.g., which ingredient the agent is interacting with), producing alignment rewards that interfere with task completion. 
Hierarchical rewards, by contrast, avoid these issues by conditioning on temporally extended options rather than raw state transitions.

\begin{figure}
  \centering
  \includegraphics[width=0.73\linewidth]{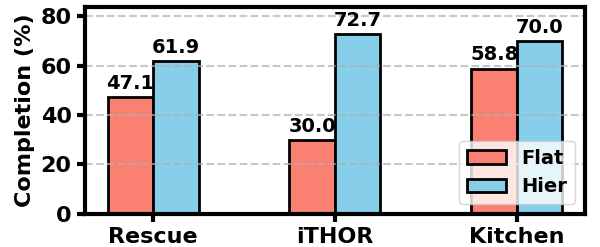}
  \caption{Task completion rates for LLM-generated rewards, calculated as the proportion of designed rewards that preserve task feasibility among syntactically valid candidates.}
  \label{fig:task-comp}
\end{figure}

\textbf{Q3}. \textbf{Do the designed rewards actually lead to agent behavior that match the behavioral specifications?}
Table~\ref{tab:sim-results} summarizes alignment performance using handcrafted high- and low-level ground-truth rewards $(\tilde{r}_H, \tilde{r}_L)$. 
The \textit{Total} metric combines task and alignment rewards and serves as a proxy for the overall fitness $F$.
In Rescue World, \textit{Hier} substantially outperforms both \textit{Task} and \textit{Flat} baselines on high-level alignment, achieving expert-level performance in 76.92\% of successful runs. 
This reflects its ability to encode the \textit{persistence} specification, which flat rewards fundamentally cannot represent. 
While \textit{Flat} occasionally (12.50\%) attains expert-level high-level alignment, these instances are coincidental.
Both methods perform comparably on low-level alignment, as the agent’s carrying status can be inferred from observable states without explicit option conditioning.
Overall, \textit{Hier} achieves the highest total return, with 69.23\% of policies attaining expert-level alignment.

In iTHOR, \textit{Hier} again outperforms \textit{Flat} on high-level alignment, as the \textit{diversity} specification cannot be fully captured by the flat reward without access to the agent's current option.
In this case, \textit{Hier} also significantly outperforms \textit{Flat} on low-level alignment as well.
Although the LLM is explicitly asked to penalize the agent's proximity to the stool when \textit{on its way} to pick up or deliver, the generated flat rewards only apply the penalty in the timestep that the agent is specifically performing the pick or place action, leading the agent to not avoid the stool. 
This makes sense, as it is difficult to discern the agent's intent in picking up or dropping an egg from just the state without option information.

In Kitchen, \textit{Hier} achieves substantially higher alignment with the chopping specification (92.86\% vs. 10.00\%). 
Although flat rewards are theoretically capable of encoding this behavior, doing so requires complex and brittle logic: only 1 flat reward candidate successfully implemented the specification.
This demonstrates a key advantage of designing rewards for HRL with HRD: even when flat rewards are theoretically sufficient, hierarchical rewards can simplify reward design through high-level abstractions and lead to better alignment with behavioral specifications.
Example videos of policies for all domains are provided in the supplementary material.

\subsection{Evaluations with Human Participants}
In real-world applications, manually designed ground-truth rewards are rarely available. 
To better reflect practical deployment scenarios, we conducted an IRB-approved user study on Rescue World and Kitchen using human participants recruited via Prolific.
The goal was to assess whether \textit{Hier} agents are perceived as better aligned with behavioral specifications than \textit{Flat} agents.

In this study, non-expert participants effectively served the role of fitness function $F$, providing human-centered evaluations.
Participants viewed videos of agent behaviors produced by both methods and rated their alignment with textual specifications on a scale from 1 (least aligned) to 5 (most aligned), similar to the scale employed in the evaluation methodology of~\cite{kwon2023reward}.
Participants were not aware of the underlying reward design methods of the policies.
We collected usable responses (e.g., those that passed attention checks) from 30 participants, evenly split across the two domains.
Further details of the study design are provided in the Appendix.

\begin{figure}
  \centering
  \includegraphics[width=0.93\linewidth]{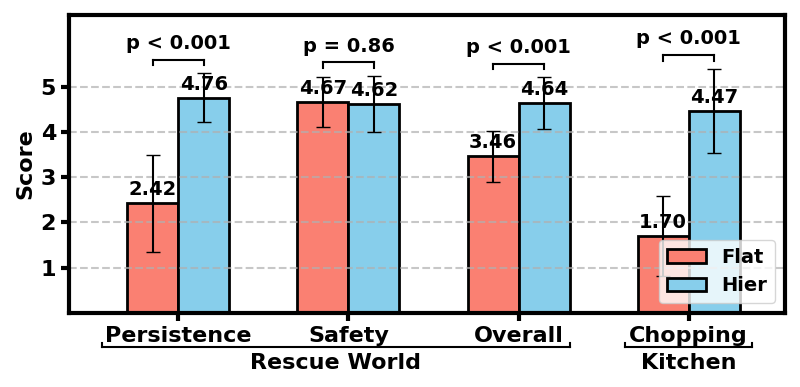}
  \caption{Human-provided ratings for agent alignment.}
  \label{fig:human-ratings}  
\end{figure}

As shown in Fig.~\ref{fig:human-ratings}, participants consistently rated \textit{Hier} agents higher than \textit{Flat} for high-level alignment. 
In Rescue World, \textit{Hier} significantly outperforms \textit{Flat} on the \textit{persistence} preference (4.76 vs. 2.42), a specification that flat rewards struggle to capture due to the lack of previous option information.
While both methods have similar ratings for the low-level \textit{safety} specification, \textit{Hier} achieves significantly higher \textit{overall} alignment scores (4.64 vs. 3.46).

In Kitchen, the advantage of \textit{Hier} is even more pronounced for the \textit{chopping} specification (4.47 vs. 1.70).
This larger gap arises because, unlike in Rescue World, the preferred behavior in Kitchen rarely occurs accidentally, 
as aligning with the \textit{chopping} specification requires taking additional steps in the environment. 
Notably, across both domains, \textit{Hier} policies consistently receive average ratings above 4, indicating a strong perception of alignment.
These findings suggest that, in practice, when task completion is used to filter out unsuccessful policies, the remaining \textit{Hier} candidates are consistently well-aligned with behavioral specifications. 

\begin{table}
  \centering
        \caption{Candidate agent policies (\%) that received a perfect alignment score from all human participants.}
        \label{tab:perfect-ratings}
        \resizebox{0.47\textwidth}{!}{
        \begin{tabular}{lcccc}
        \toprule
           & \multicolumn{3}{c}{\textbf{Rescue World}} & \textbf{Kitchen}\\
        \cmidrule(lr){2-4}\cmidrule(lr){5-5}
        \textbf{Method} & \textbf{Persistence $\uparrow$} & \textbf{Safety $\uparrow$} &
        \textbf{Overall $\uparrow$} & \textbf{Chopping $\uparrow$}\\
        \midrule
        Flat & 12.50\% & 50.00\% & 12.50\% & 10.00\%\\
        Hier & \textbf{76.92\%} & \textbf{61.54\%} &
               \textbf{53.85\%}  & \textbf{71.43\%}\\
        \bottomrule
      \end{tabular}
      }  
\end{table}
Table~\ref{tab:perfect-ratings} shows that over half of the policies produced by \textit{Hier} achieve perfect human ratings across all behavioral specifications, substantially outperforming those generated by \textit{Flat}.
Overall, \textit{Hier} consistently outperforms \textit{Flat} in capturing language-based behavioral specifications across domains in both simulated evaluations and user studies.
\section{Conclusion}
\label{sec:conclusion}
This paper introduces the \textbf{Hierarchical Reward Design (HRD) problem}, which (1) formulates a more expressive reward structure than flat rewards, (2) integrates seamlessly with existing decision-making frameworks and RL algorithms, and (3) better encodes behavioral specifications for long-horizon tasks, with an initial solution to this problem (\textbf{L2HR}) achieving considerably better or comparable results in both numerical and human evaluations.

While our results provide strong motivation for HRD, several limitations and interesting areas for future investigation remain.
First, our experiments focus on complex but simulated domains; evaluating HRD in real-world settings, such as robotics and interactive AI systems, is an important next step.
Second, more advanced reward generation techniques, including evolutionary optimization or human feedback, may further improve performance.
Finally, since option discovery remains a core challenge in Hierarchical RL, jointly addressing option discovery and reward design could substantially enhance practical applicability.

Finally, \textit{we emphasize that human-agent alignment is inherently challenging and HRD should not be viewed in isolation}.  
Rather, it complements a broader ecosystem of alignment approaches, including learning from demonstrations, rankings, and user corrections. 
Understanding how HRDL and L2HR interact with these approaches is an important avenue for future work, bringing us closer to AI agents that are reliably aligned with human needs, values and objectives.
\begin{acks}
We thank the anonymous reviewers and the members of the Ken Kennedy Institute at Rice University for their constructive feedback. This research was supported in part by NSF award $\#2205454$ and Rice University funds.
\end{acks}

\balance
\bibliographystyle{ACM-Reference-Format} 
\bibliography{refs}


\ifarxiv
\appendix
\onecolumn
\begin{center}
    \Huge \bf Appendix
\end{center}
\part*{}
\etocsetlocaltop.toc{part}
\etocsetnexttocdepth{subsection} 
\etocsettocstyle{\section*{Table of Contents}}{}
\localtableofcontents
\newpage
\section{Further Discussion of Related Work}
\label{app:related}

\textbf{Reward Machines (RMs)}
Reward Machines (RMs)~\cite{icarte2022reward} provide a structured representation of reward functions. 
However, their primary objective differs fundamentally from that of Reward Design and, consequently, HRD.
RMs aim to help RL agents exploit its reward structure to improve sample efficiency during learning. 
In contrast, HRD focuses on producing hierarchical reward structures that maximize policy fitness, which in our case is measured by the policy’s alignment with behavioral specifications.
In addition to the different goals, the underlying representations are also distinct.
RMs rely on temporal logic formulas to define structured rewards, while HRD extends the original Reward Design Problem~\cite{singh2009rewards} and defines hierarchical reward functions inspired by established HRL frameworks (e.g., options and feudal hierarchies).
Building on RMs, \cite{furelos2023hierarchies} introduced Hierarchies of Reward Machines (HRMs) to increase the expressivity of RMs and further accelerate policy convergence in some cases.
Again, while HRMs aim to improve learning efficiency, HRD focuses on enabling the capture of complex behavioral specifications through a hierarchical reward structure, particularly for long-horizon tasks.
Because RMs or HRMs do not focus specifically on reward design from human input, the ease of specifying RMs or HRMs from human input remains to be investigated. 
In contrast, for the HRD problem using the L2HR algorithm and human evaluations, we have empirically shown that language models can leverage the structure of HRD to generate hierarchical rewards that induce policies well-aligned with complex behavioral specifications.
An interesting future direction is to explore novel approaches that leverage complementary strengths of RMs and HRD to facilitate both sample-efficient and human-aligned AI.

\textbf{LLM-RL Hybrid Agents}
Although our work focuses on hierarchical RL paradigms, there is a growing body of research on hybrid agents that combine large language models (LLMs) for subtask selection with reinforcement learning for subtask execution \citep{ahn2022can, zhang2023bootstrap}.
HRD can help contribute to this line of work in two key ways. 
First, it can provide a formalism for analyzing reward design in hybrid LLM-RL systems. 
Second, it can inform the development of advanced methods that leverage LLMs with hierarchical reward structures, combining the \textit{expressivity} of hierarchical rewards with a \textit{user-friendly} reward design process. \newpage
\section{Theoretical Analysis}
\subsection{Low- and High-Level Decision Models}
\label{app:model}
\subsubsection{Low-Level MDP Models}

\setcounter{proposition}{0}
\begin{proposition}[Low-level MDP Models]
Let $\mathcal{M}_p = (\mathcal{S}, \mathcal{A}, T, \gamma)$ be a world model, $\mathcal{O}$ a set of options, and $r_L: \mathcal{S} \times \mathcal{O} \times \mathcal{A} \rightarrow \mathbb{R}$ a low‑level reward. 
For a fixed option $o \in \mathcal{O}$, the tuple $\mathcal{M}_{L, o} = (\mathcal{S}, \mathcal{A}, T, r_L(\cdot, o, \cdot), \gamma, h_o)$ defines an MDP, where $h_o$ is the horizon determined by the option’s termination condition $\beta(\cdot, o)$.
\end{proposition}

\begin{proof}
The state space is $\mathcal{S}$ and the action space is $\mathcal{A}$.
The transition function $T$ is the transition function of the world model and is Markovian.
Defining $r_{L,o}(s, a) = r_L(s, o, a)$ produces a Markov reward function that depends only on the state-action pair. 
Thus, the tuple satisfies the four standard components of an MDP.
\end{proof}

\subsubsection{High-Level SMDP Model}

\begin{proposition}[High-level SMDP Model]
Let $\mathcal{M}_p=(\mathcal{S}, \mathcal{A}, T , \gamma, h)$ be a world model, $\mathcal{O}$ a set of options, and $r_H : \mathcal{O} \times \mathcal{S} \times \mathcal{O} \rightarrow \mathbb{R}$ the high‑level reward. Then, $\mathcal{M}_H = (\mathcal{O} \times \mathcal{S}, \mathcal{O}, T_H, r_H, \gamma, h)$ forms a semi-MDP, where $T_H: \mathcal{O} \times \mathcal{S} \times \mathcal{O} \rightarrow \Delta(\mathcal{O} \times \mathcal{S} \times \mathbb{N})$ defines the joint distribution over the next augmented state and transit time, where $\mathbb{N}$ is the set of natural numbers.
\end{proposition}

\begin{proof}
A semi-MDP (SMDP) requires (1) a set of states, (2) a set of actions, (3) for each state-action pair, an expected discounted reward, and (4) a well-defined joint distribution over the next state and transit time \cite{sutton1999between}.

We define the state space as the Cartesian product $\mathcal{O} \times \mathcal{S}$, combining the previous option $o^- \in \mathcal{O}$ and the current environment state $s \in S$.
The action space is the set of options $\mathcal{O}$.
The reward function for each state-action pair $\big((o^-, s), o\big)$ is given by the provided high-level reward $r_H(o^-, s, o)$.

To define the transition dynamics, consider the transition probability from a given augmented state $(o^-, s)$ upon selecting an option $o$.
Let $s'$ denote the state upon option termination and $k$ the number of timesteps to reach $s'$.
The transition function $T_H$ is defined as:
\begin{align}
T_H(o, s', k | o^-, s, o) = \sum_{\tau: \ (o^-, s) \rightarrow (o, s')} \text{Pr}(\tau; \pi_{L, o}) \cdot \mathbb{I}_{\{k = |\tau|\}} \cdot \beta(s', o)
\end{align}
where $\tau$ is any trajectory that starts in state $s$ after option $o^-$ and reaches state $s'$ via option $o$, $\text{Pr}(\tau; \pi_{L, o})$ is the probability of trajectory $\tau$ under policy $\pi_{L, o}$, and $|\tau|$ is the number of timesteps taken by the trajectory.
Given a specific $\tau = (o_{t-1}, s_t, o_t, s_{t+1}, \ ... \ , o_{t+\eta-1}, s_{t+\eta})$ and knowing that $o_{t-1} = o^-$ and $o_t, \ ... \ , o_{t+\eta-1} = o$, the probability of the trajectory is given by:
\begin{align}
\text{Pr}(\tau; \pi_{L, o}) = \prod_{i=0}^{\eta-1} \Big(\sum_a \pi_{L, o} (a|s_{t+i}) T(s_{t+i+1}|s_{t+i}, a) \Big) \cdot \prod_{i=0}^{\eta-2}\Big(1-\beta(s_{t+i+1}, o)\Big)
\end{align}
The first product accounts for the probabilities of transitioning through the intermediate states under $\pi_{L, o}$ and the environment's underlying dynamics $T$.
The second product ensures that the option does not terminate at intermediate states prior to reaching $s'$.
As all four conditions are satisfied, the tuple $\mathcal{M}_H$ forms a valid SMDP.
\end{proof}

\subsubsection{High-Level MDP Model}
Alternatively, the high-level process can be modeled as an MDP if \textit{single-step} high-level rewards $r_H^{step}(o_{t-1}, s_t, o_t)$ are defined. 
\cite{zhang2019dac} was the first to model high-level decision-making as an MDP within the options framework.
While our formulation differs from theirs, we draw inspiration from their use of \textit{single-step} high-level rewards and demonstrate that the high-level process in our setting can also be modeled as an MDP.

\begin{proposition}[High-level MDP Model]
Let $\mathcal{M}_p=(\mathcal{S}, \mathcal{A}, T , \gamma, h)$ be a world model, $\mathcal{O}$ a set of options, and $r_H^{step} : \mathcal{O} \times \mathcal{S} \times \mathcal{O} \rightarrow \mathbb{R}$ a ``single-step'' high‑level reward. 
Then, $\mathcal{M}_H^{step} = (\mathcal{O} \times \mathcal{S}, \mathcal{O}, T_H^{step}, r_H^{step}, \gamma, h)$ forms an MDP, where $T_H^{step}: \mathcal{O} \times \mathcal{S} \times \mathcal{O} \rightarrow \Delta(\mathcal{O} \times \mathcal{S})$.
\end{proposition}
\begin{proof}
We again define the state space as $\mathcal{O} \times \mathcal{S}$ and the action space as $\mathcal{O}$.
The transition function is defined as:
$T_H^{step}(o, s' | o^-, s, o ) = \sum_{a} \pi_L(a | s, o) \cdot  T(s' | s, a) \cdot \mathbb{I}_{\{o' = o\}}$.
Termination condition is not modeled in the transition function, since option selection occurs at every step.
The reward function is $r_H^{step}$, which satisfies the MDP requirements.
\end{proof}

As a side note, we can derive the expected SMDP high-level reward from the single-step rewards as follows:
\begin{align}
r_H(o^-, s, o) \doteq \mathbb{E}[\sum_{i=1}^k \gamma^{i-1} {r_H^{step}}_{t+i} | \mathcal{E}(o^-, o, s, t)]
\label{eq:highr-until-termination}
\end{align}
where $\mathcal{E}(o^-, o, s, t)$ is the event of initiating option $o$ in state $s$ at timestep $t$ following option $o^-$, and $k$ is the random variable denoting the number of steps after which the initiated option $o$ terminates, as determined by its termination condition $\beta$.

\subsection{Policy Learning with Hierarchical Rewards}
\label{app:learn}
Given hierarchical reward functions in \cref{def: low-level reward} and \cref{def: high-level reward}, agents can learn low- and high-level policies through interactions with the environment under their respective decision-making models.

\subsubsection{Low-Level Policy Learning}
For each option $o_i \in \mathcal{O}$, the low-level decision-making process can be formulated as an MDP $\mathcal{M}_{L, o_i}$ (see \cref{prop:low-level}).
Let $r_{L, o_i}(s, a) = r_L(s, o_i, a)$ represent the reward function for option $o_i$ and $\pi_{o_i}(a|s)$ a policy in the MDP.
The objective for low-level policy learning is to maximize the cumulative discounted rewards until the termination of option $o_i$:%
\begin{align}
\pi_L(a|s, o\myeq o_i) &= \arg\max_{\pi_{o_i}} \mathbb{E}_{s_t \sim T, a_t \sim \pi_{o_i}}[\sum_{t=0}^{h_o} \gamma^t r_{L, o_i}(s_t, a_t)| \mathcal{M}_{L, o_i}]
\label{eq: low-level HRD obj}
\end{align}%
where $h_o$ is the horizon determinied by the option's termination condition $\beta(\cdot, o_i)$. 
Standard RL algorithms for MDPs can be directly used to obtain low-level policies.

\subsubsection{High-Level Policy Learning (SMDP)}
When the high-level decision-making is modeled by an SMDP, the high-level policy $\pi_H(o|o^-, s)$ selects options only at the termination of the previous option, operating on a coarser temporal scale compared to the low-level policy $\pi_L(a|s, o)$.
Let $u = 0, \dots, h$ index the \textit{high-level} decision points, and define $\eta_{o_u}$ as the number of primitive timesteps taken to execute option $o_u$. 
Then, the high-level policy learning objective is:%
\begin{align} \label{eq: high-level-obj-smdp}
\pi_H(o|o^-, s) = \arg\max_{\pi} \mathbb{E}_{(o_{u-1}, s_u, \eta_{o_{u-1}}) \sim T_H, \, o_u \sim \pi} &[\sum_{u=0}^{h} \gamma^{T_u} r_H(o_{u-1}, s_u, o_u) | \mathcal{M}_H], \\
\text{where } T_u &= \sum_{j=-1}^{u-1} \eta_{o_j} \notag
\end{align}%

Here, we need to extend the SMDP by introducing a dummy initial option $o_{\#}$, with zero duration, and let $o_{-1} = o_{\#}$.
Any algorithm for learning option-level policies in an SMDP can be used here. 

\subsubsection{High-Level Policy Learning (MDP)}
The learning objective when the high-level decision-making is modeled by an MDP is:%
\begin{align}
\pi_H^{step}(o|o^-, s) = \arg\max_{\pi} \mathbb{E}_{(o_{t-1}, s_t) \sim T_H^{step}, \, o_t \sim \pi}[\sum_{t=0}^h \gamma^{t} r_H^{step}(o_{t-1}, s_t, o_t))| \mathcal{M}_H^{step}]
\label{eq: high-level-obj-mdp}
\end{align}%
Similarly, we let $o_{-1} = o_{\#}$.
When the termination condition is known, as in our setting, the high-level policy can be expressed as: $\pi_H^{step}(o|o^-, s) \doteq \beta(s, o^-) \pi_H(o|o^-, s) + (1-\beta(s, o^-)) \mathbb{I}_{(o = o^-)}$, where $\pi_H(o|o^-, s)$ specifies the policy at decision points.
This allows leveraging known termination conditions to constrain policy rollouts.

\subsection{Expressivity of HRD}
\label{app:expressivity}
\setcounter{property}{0}
\begin{property}
Certain specifications on \textit{sub-task selection} can be expressed through $\tilde{r}_H(o^-, s, o)$, but they cannot be expressed by a flat reward function: $\tilde{r}_{flat}(s, a)$. 
\end{property}
\begin{proof}
Consider two distinct subtask sequences, \{$o_1 \rightarrow o_2$\} and \{$o_1' \rightarrow o_2'$\} such that after executing either $o_1$ or $o_1'$, the agent arrives at the same state $s^*$.
The designer’s specification is for the agent to follow the corresponding subtask sequences: execute $o_2$ after $o_1$ and $o_2'$ after $o_1'$.
A flat reward function $\tilde{r}_{flat}(s, a)$ cannot represent this specification, as the reward signal depends only on the current state and action, and cannot distinguish whether the agent reached $s^*$ via $o_1$ or $o_1'$.
In contrast, with HRD's high-level reward function, we can specify the ordering even in $s^*$ by defining $\tilde{r}_H(o^-\myeq o_1, s\myeq s^*, o\myeq o_2)>\tilde{r}_H(o^-\myeq o_1, s\myeq s^*, o\neq o_2)$ and $\tilde{r}_H(o^-\myeq o_1', s\myeq s^*, o\myeq o_2')>\tilde{r}_H(o^-\myeq o_1', s\myeq s^*, o\neq o_2')$.
\end{proof}
    
\begin{property} 
Certain specifications on \textit{sub-task execution} can be expressed through $\tilde{r}_L(s, o, a)$, but they cannot be expressed by a flat reward function: $\tilde{r}_{flat}(s, a)$. 
\end{property}

\begin{proof}
Consider a setting where the behavioral specification explicitly depends on the current option $o$.
Such specifications cannot be represented using a flat reward function, as the reward signal $\tilde{r}_{flat}(s, a)$ is identical for all option values. 
\end{proof}

\vfill \newpage
\section{\algofullname(\algo) Pseudocode}
\label{app:pseudocode}
\begin{algorithm}[H]
\caption{\algofullname~(\algo)}
\begin{algorithmic}[1]
\State \textbf{Input:} World model $\mathcal{M}_p$, set of options $\mathcal{O}$, learning routine $\mathcal{A}_{\mathcal{M}_p}$, task reward $r$, pseudo-reward $r_p$, specifications $l$, thresholds for successful subtask and task completion $t_L$ and $t_H$, and an LLM.
\State $\tilde{r}_L^{(1)}, \dots, \tilde{r}_L^{(k)} \leftarrow$ LLM(\Call{LowLevelPrompt}{$l$}) \Comment{Generate low-level alignment rewards}
\For{$i = 1$ to $k$}
    \If{Static\_Check($\tilde{r}_L^{(i)}$)}
    \Statex \hspace{3em} \textit{// $\mathcal{A}_{\mathcal{M}_{p: \, L, \cdot}}$ denotes the learning subroutine for low-level MDPs $\mathcal{M}_{L, \cdot}$}
    \State $\pi_L^{(i)} \leftarrow \mathcal{A}_{\mathcal{M}_{p: \, L, \cdot}}(r_p + \tilde{r}_L^{(i)})$ 
    \EndIf
\EndFor
\State $\mathcal{I} \leftarrow$ Indices of $\pi_L^{(i)}$ achieving cumulative pseudo-rewards above threshold $t_L$
\State $\tilde{r}_H^{(1)}, \dots, \tilde{r}_H^{(k)} \leftarrow$ LLM(\Call{HighLevelPrompt}{$l$}) \Comment{Generate high-level alignment rewards}
\For{$j = 1$ to $k$}
    \If{Static\_Check($\tilde{r}_H^{(j)}$)}
        \State Select low-level policy $\pi_L^{(i)}$, where $i \in \mathcal{I}$ \Comment{Done via hashing in our implementation}
        \Statex \hspace{3em} \textit{// $\mathcal{A}_{\mathcal{M}_{p: \, H}}$ denotes the learning subroutine for the high-level model $\mathcal{M}_H$}
        \State $\pi_H^{(j)} \leftarrow \mathcal{A}_{\mathcal{M}_{p: \, H}}(r + \tilde{r}_H^{(j)}; \pi_L^{(i)})$
    \EndIf
\EndFor
\State \textbf{Output:} Return all alignment rewards $(\tilde{r}_H^{(j)}, \tilde{r}_L^{(i)})$ and corresponding trained policy pairs $(\pi_H^{(j)}, \pi_L^{(i)})$ that achieve cumulative task rewards above threshold $t_H$
\end{algorithmic}
\label{alg:train}
\end{algorithm} \newpage
\section{Experimental Details}
\label{app:exp}
\subsection{Further Domain Details}
\label{app:domains}
\textbf{Rescue World} is a variant of the RW4T domain~\citep{orlov2024rw4t}, a configurable testbed for simulating disaster response scenarios in which a first responder deploys robots to collect scattered supplies and deliver them to designated areas.
For our experiments, we configure the environment with a single robot and represent the world using a discrete grid-based layout (\cref{fig:rescue_world2}).
The robot must determine both \textit{the delivery order} of supplies and \textit{the optimal path} for completing deliveries.
This setting naturally motivates a hierarchical action representation due to the need for subtask sequencing and execution.
The primitive action space $\mathcal{A}$ includes six actions: pick, drop, and four directional movements (up, down, left, right).
The option space $\mathcal{O}$ consists of multi-step pick-up and delivery macro-actions for two supply types: \textit{food} and \textit{medical supplies}.

\textbf{iTHOR} is a simulator built within the AI2-THOR framework~\cite{kolve2017ai2}, featuring realistic household environments where an agent can navigate and interact with everyday objects.
We use \textit{FloorPlan 20} as the environment for our experiments.
In our setup, apples and eggs are spawned on one side of a long kitchen table, the sink is located on the opposite side, and a stool is positioned to the right of the table, as shown in ~\cref{fig:ithor}.
The option space $\mathcal{O}$ consists of pick-up and place operations for each object type, while the action space $\mathcal{A}$ includes navigation actions (move forward, turn left, turn right) and pick/place primitives.
While conducting experiments with expert-provided rewards (Sec.~\ref{app:proofofconcept}), we observe that inducing the agent to follow the \textit{diversity} preference is highly sensitive to the reward scale.
To address this, we include additional contextual information in the LLM prompt specifically about the typical length of the options to guide more consistent reward generation.

\textbf{Kitchen} is a single-agent variant of Overcooked, a benchmark environment originally designed for studying human-AI collaboration in kitchen tasks~\cite{liu2023llm, wang2025zsc}.
In our setting, the agent must prepare a salad using lettuce, tomatoes, and onions.
We adopt the structured option space from~\cite{liu2023llm}, which includes high-level options such as \texttt{chop onion} and \texttt{combine chopped onion and chopped tomatoes}.
The environment also provides hard-coded low-level controllers for executing these options.
The primitive action space $\mathcal{A}$ includes directional movement actions that enable interactions with countertops and other kitchen objects.

\begin{figure}[t]
    \centering
    \begin{subfigure}[b]{0.3\textwidth}
        \centering
        \includegraphics[height=14em]{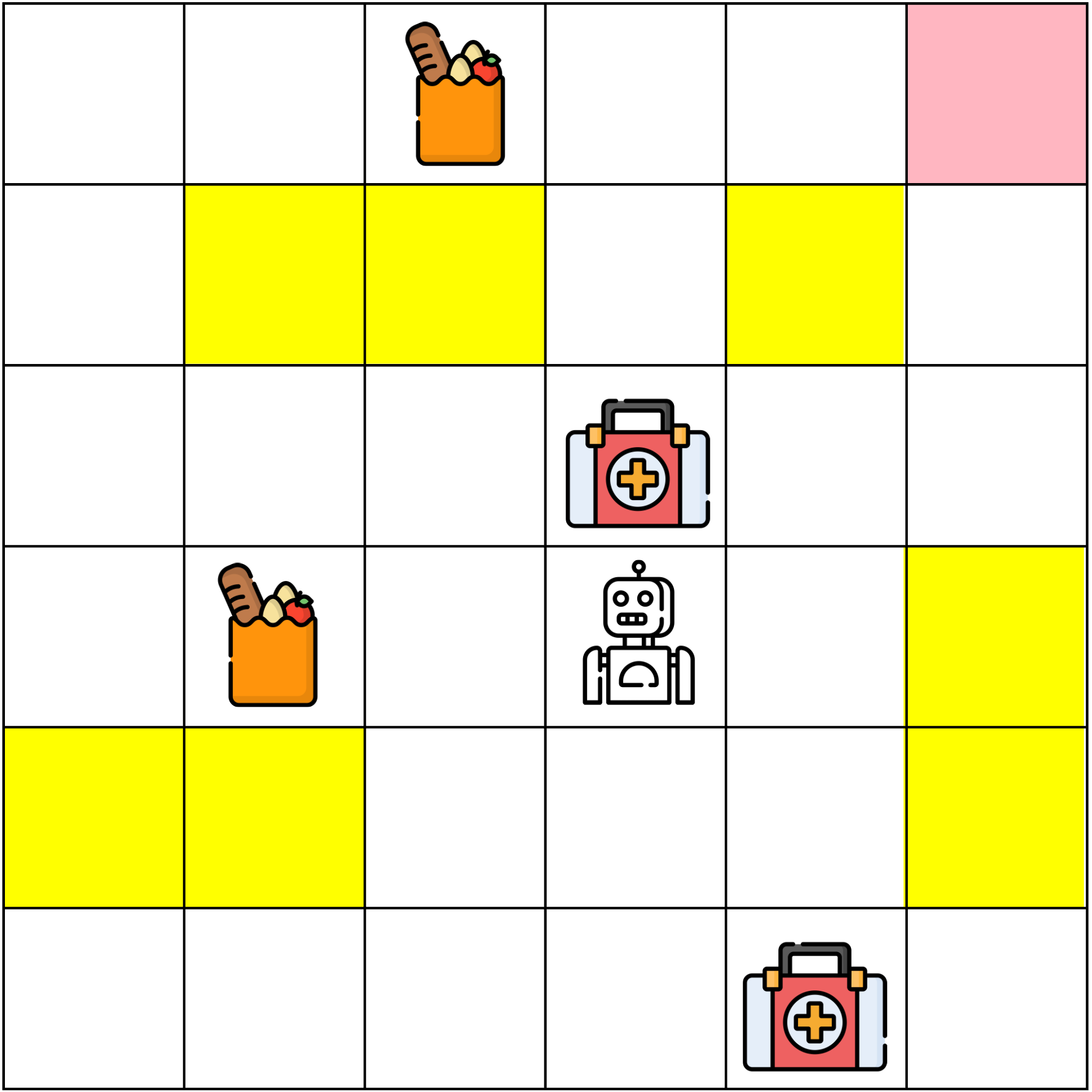} 
        \caption{Rescue World Domain}
        \label{fig:rescue_world2}
    \end{subfigure}
    \hfill
    \begin{subfigure}[b]{0.3\textwidth}
        \centering
        \includegraphics[height=14em]{figures/thor_pnp.png} 
        \caption{iTHOR Domain}
        \label{fig:ithor2}
    \end{subfigure}
    \hfill
    \begin{subfigure}[b]{0.3\textwidth}
        \centering
        \includegraphics[height=14em]{figures/kitchen.png} 
        \caption{Kitchen Domain}
        \label{fig:kitchen2}
    \end{subfigure}
    \caption{Screenshots of renderings of the three task domains used in our study.}
    \label{fig:domain-viz2}
\end{figure}

The visualizations of the three task domains used in our study are shown in Fig.~\ref{fig:domain-viz2}.
The Rescue World visualization was generated using Pygame based on its Gym environment, the iTHOR visualization was created by visually rendering the environment within the Unity engine, and the Kitchen domain visualization was adopted from \cite{liu2023llm}.

\subsection{Results with Expert-Provided Rewards}
\label{app:proofofconcept}

\begin{table*}[t]
\centering
\caption{
Table showing the performance of policies trained with the task reward alone and with task reward combined with expert-provided flat or hierarchical rewards.
For each metric (``High-Level,'' ``Low-Level,'' and ``Total''), we report both the cumulative returns and the percentage of policies achieving expert-level alignment.
A policy is considered expert-aligned at the high or low level if it attains the maximum possible cumulative return for that metric.
``Total'' represents the sum of the task reward, high-level alignment reward, and low-level alignment reward, and a policy is deemed expert-level overall if it aligns at both levels.
}
\label{tab:sim-results-gt}
\begin{center} 
\resizebox{0.95\textwidth}{!}{
        \begin{tabular}{llcccccc}
            \toprule
            \multirow{2}{*}{\textbf{Domain}} & 
            \multirow{2}{*}{\textbf{Method}} & 
            \multicolumn{2}{c}{\textbf{High-Level}} & 
            \multicolumn{2}{c}{\textbf{Low-Level}} & 
            \multicolumn{2}{c}{\textbf{Total}} \\
            \cmidrule(lr){3-4}\cmidrule(lr){5-6}\cmidrule(lr){7-8} &  & \textbf{Rewards $\uparrow$} & \textbf{ Expert \% $\uparrow$} & \textbf{Rewards $\uparrow$} & \textbf{Expert \% $\uparrow$} & \textbf{Rewards $\uparrow$} & \textbf{Expert \% $\uparrow$}
            \\ \midrule
            \multirow{3}{*}{Rescue} & Task & 11.22 $\pm$ 5.57 & 20.00 & -16.46 $\pm$ 5.49 & 0.00 & 73.80 $\pm$ 5.70 & 0.00 \\
            & Flat* & 6.84 $\pm$ 5.24 & 0.00 & \textbf{-0.32 $\pm$ 0.72} & \textbf{80.00} & 85.88 $\pm$ 4.66 & 0.00 \\
            & Hier* & \textbf{20.00 $\pm$ 0.00} & \textbf{100.00} & -1.00 $\pm$ 2.24 & 80.00 & \textbf{97.15 $\pm$ 2.94} & \textbf{80.00} \\
            \midrule
            \multirow{3}{*}{iTHOR} & Task & 4.10 $\pm$ 1.34 & 0.00 & -23.38 $\pm$ 1.86 & 0.00 & 12.31 $\pm$ 0.66 & 0.00 \\
            & Flat* & 7.80 $\pm$ 2.41 & 0.00 & \textbf{0.00 $\pm$ 0.00 }& \textbf{100.00*} & 35.41 $\pm$ 2.98 & 0.00 \\
            & Hier* & \textbf{15.00 $\pm$ 0.00} & \textbf{100.00} & \textbf{0.00 $\pm$ 0.00} & \textbf{100.00} & \textbf{42.61 $\pm$ 0.35} & \textbf{100.00} \\
            \midrule
            \multirow{3}{*}{Kitchen} & Task & 0.00 $\pm$ 0.00 & 0.00 & -- & -- & 0.75 $\pm$ 0.00 & 0.00 \\
            & Flat* & \textbf{0.40 $\pm$ 0.00} & \textbf{100.00} & -- & -- & \textbf{1.10 $\pm$ 0.00} & \textbf{100.00} \\
            & Hier* & \textbf{0.40 $\pm$ 0.00} & \textbf{100.00} & -- & -- & \textbf{1.10 $\pm$ 0.00} & \textbf{100.00} \\
            \bottomrule
        \end{tabular}
        } 
    \end{center}
\end{table*}

\cref{tab:sim-results-gt} reports the results of policies trained with expert-provided flat $\tilde{r}_{flat}^*$ and hierarchical rewards $(\tilde{r}_H^*, \tilde{r}_L^*)$ (shown in the table as \textit{Flat*} and \textit{Hier*} respectively).
In other words, no LLMs were used for reward generation in these experiments.
Across all three domains (Rescue World, iTHOR, and Kitchen), \textit{Hier*} consistently achieves the highest total returns.

In Rescue World, \textit{Hier*} matches \textit{Flat*} in low-level alignment return but significantly outperforms it in high-level alignment, achieving an average total return of 97.15 with 80.00\% policies reaching expert alignment, compared to 85.88 and 0.00\% for \textit{Flat*}.
A similar pattern is observed in iTHOR, where \textit{Hier*} achieves an average total return of 42.61 with all policies reaching expert alignment, while \textit{Flat*} achieves 35.41 with none reaching expert alignment. 
We also note that although \textit{Flat*} matched \textit{Hier*} in low-level alignment for iTHOR, it did so by following the low-level \textit{avoidance} preference regardless of the option.
In contrast, \textit{Hier*} selectively applies this preference to options involving eggs, demonstrating finer-grained alignment with low-level human specifications. 
These results highlight that \textit{Hier*} better captures behavioral dependencies related to the current and previous options, enabling it to represent high-level \textit{persistence} in Rescue World and \textit{diversity} in iTHOR, which are beyond the representational capacity of flat rewards.

In Kitchen, \textit{Hier*} and \textit{Flat*} perform equally well on both task and high-level returns. 
While \textit{Flat*} is theoretically capable of capturing dependencies on the previous subtask in Kitchen by inferring it from environment state, doing so is difficult and error-prone without expert-designed rewards, as shown in Sec.~\ref{sec:experiments}.
Overall, these results demonstrate that expert-designed hierarchical rewards can be easily integrated with task-related rewards to train policies that (1) achieve strong task performance comparable to baselines without reward design (2) align well with behavioral specifications, including those that flat rewards cannot effectively represent.

\subsection{Training Setup}
In our implementation, we used LLMs to generate \textit{single-step} high-level rewards, allowing the high-level decision process to be modeled as either an MDP or an SMDP.
When using an SMDP, we computed the corresponding SMDP rewards using Eq.~\ref{eq:highr-until-termination}. 

\subsubsection{Rescue World}
For Rescue World, we modeled the high-level decision-making as an SMDP and trained the high-level policy $\pi_H$ using DQN~\cite{mnih2015human}.
The hyperparameters for training $\pi_H$ are as follows:
\begin{itemize}
    \item Network: 2 layers with 64 units each and ReLU non-linearities
    \item Optimizer: Adam \citep{kingma2014adam}
    \item Learning rate: $1 \cdot 10^{-4}$
    \item Batch size: 256
    \item Discount: 1.0
    \item Total timesteps: $3 \cdot 10^6$
    \item Buffer size: $1 \cdot 10^6$
    \item Exploration fraction: 0.2
    \item Initial exploration probability: 0.1
    \item Final exploration probability: 0.05
    \item Model update frequency: 4
    \item Number of gradient steps per rollout: 1
    \item Target update interval: $1 \cdot 10^4$
    \item Polyak-averaging \citep{polyak1992acceleration}: 1.0
\end{itemize}

We trained the low-level policy $\pi_L$ using PPO \cite{schulman2017proximal}.
The hyperparameters for training $\pi_L$ are as follows:
\begin{itemize}
    \item Network: 2 layers with 64 units each and ReLU non-linearities
    \item Optimizer: Adam \citep{kingma2014adam}
    \item Learning rate: $3 \cdot 10^{-4}$
    \item Batch size: 64
    \item Discount: 1.0
    \item Total timesteps: $2 \cdot 10^6$
    \item Initial entropy coefficient: 1
    \item Final entropy coefficient: 0.01
    \item Entropy decay fraction: 0.5
    \item Number of environment steps per update: 2048
\end{itemize}

\subsubsection{iTHOR}
For iTHOR, we modeled the high-level decision-making as an SMDP and trained the high-level policy $\pi_H$ using DQN~\cite{mnih2015human}.
The hyperparameters for training $\pi_H$ are as follows:
\begin{itemize}
    \item Network: 2 layers with 128 units each and ReLU non-linearities
    \item Optimizer: Adam \citep{kingma2014adam}
    \item Learning rate: $1 \cdot 10^{-4}$
    \item Batch size: 32
    \item Discount: 0.99
    \item Total timesteps: $5.0 \cdot 10^5$
    \item Buffer size: $5 \cdot 10^5$
    \item Exploration fraction: 0.25
    \item Initial exploration probability: 1.0
    \item Final exploration probability: 0.05
    \item Model update frequency: 4
    \item Number of gradient steps per rollout: 1
    \item Target update interval: $1 \cdot 10^4$
    \item Polyak-averaging \citep{polyak1992acceleration}: 1.0
\end{itemize}

We trained the low-level policy $\pi_L$ using PPO \cite{schulman2017proximal}.
The hyperparameters for training $\pi_L$ are as follows:
\begin{itemize}
    \item Network: 2 layers with 64 units each and ReLU non-linearities
    \item Optimizer: Adam \citep{kingma2014adam}
    \item Learning rate: $3 \cdot 10^{-4}$
    \item Batch size: 64
    \item Discount: 1.0
    \item Total timesteps: $1.5 \cdot 10^6$
    \item Initial entropy coefficient: 1
    \item Final entropy coefficient: 0.01
    \item Entropy decay fraction: 0.5
    \item Number of environment steps per update: 2048
\end{itemize}

\subsubsection{Kitchen}
For Kitchen, we modeled the high-level decision-making as an MDP and trained the high-level policy $\pi_H$ using DQN~\cite{mnih2015human}, implemented so that termination conditions $\beta$ were enforced during rollouts.
We adopted the MDP formulation because it outperformed the SMDP setting with DQN in this domain.
Given the highly delayed rewards and the importance of subtask sequencing in Kitchen, we also incorporated specification-agnostic demonstrations to bootstrap policy learning.
Hyperparameters for learning the high-level policy $\pi_H$ are as follows:
\begin{itemize}
    \item Network: 2 layers with 256 units each and ReLU non-linearities
    \item Optimizer: Adam \citep{kingma2014adam}
    \item Learning rate: $1 \cdot 10^{-6}$ with linear scheduling
    \item Batch size: 256
    \item Discount: 0.99
    \item Total timesteps: $3 \cdot 10^6$
    \item Buffer size: $1 \cdot 10^6$
    \item Exploration fraction: 0.33
    \item Initial exploration probability: 0.5
    \item Final exploration probability: 0.1
    \item Model update frequency: 4
    \item Number of gradient steps per rollout: 1
    \item Target update interval: $1 \cdot 10^4$
    \item Polyak-averaging \citep{polyak1992acceleration}: 1.0
\end{itemize}

For Rescue World and Kitchen, we used a server with 30 vCPUs and an NVIDIA A10 GPU (24GB PCIe) to train $k$ policies in parallel, each corresponding to one of the $k$ reward candidates generated by the LLM. For iTHOR, we used a server with an NVIDIA GeForce RTX 5090 GPU to train our policies.

\subsection{LLM Prompts}
\label{app:prompts}
The prompts used in our work are adapted from \cite{ma2023eureka}, but differ in important ways to realize hierarchical rewards. 
Specifically, our prompts are designed to:
(1) reflect a hierarchical reward structure; and
(2) generate rewards that align with behavioral specifications while preserving task feasibility.

\subsubsection{System Prompt}
As in \cite{ma2023eureka}, our system prompt provides a concise, domain-agnostic description of the reward design task and defines the function signature that the LLM should use in its output. 
The full system prompt is provided in Prompt~1 on \cpageref{box:prompt_sys}.

\subsubsection{User Prompt}
The user prompts follow a similar methodology to that of \cite{ma2023eureka}, using code snippets as contextual input and an accompanying task-specific natural language description.
To support hierarchical reward generation, we extend the accompanying task description by providing the following: (1) a description of relevant action spaces (i.e., the option space $\mathcal{O}$ and/or action space $\mathcal{A}$) to help the LLM distinguish between temporally extended behaviors and primitive actions; (2) a behavioral specification describing preferences beyond task completion; (3) additional code formatting guidelines to emphasize that the LLM should capture behavioral logic without making the reward function stateful (e.g., storing variables across calls).
Moreover, the complexity of our domains introduces two additional challenges, which we address by augmenting the code snippets with further contextual information:

\textbf{Cross-file Dependencies.}
In our environments, key components of the task logic often depend on constants and definitions from separate supporting files (e.g., \texttt{utils.py}).
To address this, we manually copied the necessary definitions from these files and included them as background comments at the top of the environment code provided to the LLM.
This ensures all relevant constants and definitions are explicitly exposed during reward generation.

\textbf{Complex Observation Representations.}
Our environments feature structured observations, such as spatial maps, whose semantics are not fully captured by the observation's shape or naming alone.
For example, it can be difficult to infer what each value in the observation (e.g., map cell) represents from the environment code.
To mitigate this, we also provided an example observation input as part of the background comments in cases where the LLM might find it challenging to correctly interpret the structure and meaning of the observation space.

To ensure a fair comparison across conditions, the same code snippets of each domain were used for all reward generation tasks, whether generating low-level, high-level, or flat rewards.
The full environment contexts and corresponding prompts for both Rescue World and Kitchen are shown in Prompt 2 and Prompt 3 on \cpageref{box:prompt_rw} and \cpageref{box:prompt_oc} respectively.

\begin{tcolorbox}[title={Prompt 1: System Prompt},breakable, label={box:prompt_sys}]
You are a reward engineer trying to write reward functions to solve reinforcement learning tasks as effective as possible.
A programmer has already specified the task reward, and your job is to specify additional rewards according to the user's preference.
More specifically, your goal is to write an additional reward function for the environment to help the agent complete the task according to user preference. 
Your reward function should use useful variables from the environment as inputs. As an example,
the reward function signature can be:

\bigskip
\textit{The LLM is presented with one of the following function signatures, selected based on the desired reward function to be designed. }
\begin{minted}[fontsize=\small,baselinestretch=1]{python}
def get_high_level_pref_gpt(state: Dict, prev_option: int, option: int) -> 
    Tuple[float, Dict[str, float]]:
    '''
    state: the current state of the environment.
    prev_option: the last option (subtask) executed by the agent to reach the 
    current state.
    option: the option (subtask) the agent is about to perform in the current 
    state.
    '''
    ...
    return reward, reward_components
\end{minted}

\smallskip
\begin{minted}[fontsize=\small,baselinestretch=1]{python}
def get_low_level_pref_gpt(state: Dict, option: int, action: int) -> 
    Tuple[float, Dict[str, float]]:
    '''
    state: the current state of the environment.
    option: the option (subtask) selected by the agent in the current state.
    action: the action that the agent is about to perform in the current state.
    '''
    ...
    return reward, reward_components
\end{minted}

\smallskip
\begin{minted}[fontsize=\small,baselinestretch=1]{python}
def get_flat_sa_pref_gpt(state: Dict, action: int) -> Tuple[float, 
    Dict[str, float]]:
    '''
    state: the current state of the environment.
    action: the (low-level) action that the agent is about to perform in the 
    current state.
    '''
    ...
    return reward, reward_components
\end{minted}

The output of the reward function should consist of two items: \\
(1) the user preference reward, \\
(2) a dictionary of each individual reward component in the user preference reward. \\
The code output should be formatted as a python code string: "python ... ".

Some helpful tips for writing the reward function code: \\ 
(1) Most importantly, the reward code's input variables must contain only attributes of the provided environment class definition (namely, variables that have prefix self.). Under no circumstance can you introduce new input variables.
\end{tcolorbox}

\begin{tcolorbox}[title={Prompt 2: User Prompt for Rescue World},breakable, label={box:prompt_rw}]
The Python environment is

\begin{minted}[fontsize=\small,baselinestretch=1]{python}
'''
Background:

1) Initial game map example
init_map = np.array(
    [[
        rw4t_utils.RW4T_State.empty.value, 
        rw4t_utils.RW4T_State.empty.value,
        rw4t_utils.RW4T_State.circle.value, 
        rw4t_utils.RW4T_State.empty.value,
        rw4t_utils.RW4T_State.yellow_zone.value,
        rw4t_utils.RW4T_State.school.value
    ],
     [
         rw4t_utils.RW4T_State.empty.value,
         rw4t_utils.RW4T_State.yellow_zone.value,
         rw4t_utils.RW4T_State.yellow_zone.value,
         rw4t_utils.RW4T_State.empty.value,
         rw4t_utils.RW4T_State.yellow_zone.value,
         rw4t_utils.RW4T_State.empty.value
     ],
     [
         rw4t_utils.RW4T_State.empty.value, 
         rw4t_utils.RW4T_State.empty.value,
         rw4t_utils.RW4T_State.empty.value, 
         rw4t_utils.RW4T_State.square.value,
         rw4t_utils.RW4T_State.yellow_zone.value,
         rw4t_utils.RW4T_State.empty.value
     ],
     [
         rw4t_utils.RW4T_State.empty.value, 
         rw4t_utils.RW4T_State.circle.value,
         rw4t_utils.RW4T_State.empty.value, 
         rw4t_utils.RW4T_State.empty.value,
         rw4t_utils.RW4T_State.empty.value, 
         rw4t_utils.RW4T_State.empty.value
     ],
     [
         rw4t_utils.RW4T_State.yellow_zone.value,
         rw4t_utils.RW4T_State.yellow_zone.value,
         rw4t_utils.RW4T_State.empty.value, 
         rw4t_utils.RW4T_State.empty.value,
         rw4t_utils.RW4T_State.yellow_zone.value,
         rw4t_utils.RW4T_State.yellow_zone.value
     ],
     [
         rw4t_utils.RW4T_State.empty.value, 
         rw4t_utils.RW4T_State.empty.value,
         rw4t_utils.RW4T_State.empty.value, 
         rw4t_utils.RW4T_State.empty.value,
         rw4t_utils.RW4T_State.square.value, 
         rw4t_utils.RW4T_State.empty.value
     ]])

2) rw4t.utils:
class RW4T_LL_Actions(Enum):
  go_left = 0
  go_down = 1
  go_right = 2
  go_up = 3
  pick = 4
  drop = 5
  idle = 6


class RW4T_HL_Actions_EZ(Enum):
  go_to_circle = 0
  deliver_circle = 1
  go_to_square = 2
  deliver_square = 3


class RW4T_HL_Actions_With_Dummy_EZ(Enum):
  go_to_circle = 0
  deliver_circle = 1
  go_to_square = 2
  deliver_square = 3
  dummy = 4


class RW4T_State(Enum):
  empty = 0
  circle = 1
  square = 2
  triangle = 3
  obstacle = 4
  yellow_zone = 5
  orange_zone = 6
  red_zone = 7
  school = 8
  hospital = 9
  park = 10


class Holding_Obj(Enum):
  empty = 0
  circle = 1
  square = 2
  triangle = 3
'''

import numpy as np
import gymnasium as gym

import rw4t.utils as rw4t_utils


class RW4T_GameState:

  def __init__(self, obs: np.ndarray, pos: np.ndarray, holding: int,
               option_mask: np.ndarray):
    '''
    :param obs: a 2D numpy of the current environment
    :param pos: a 1D numpy array of the agent's (x, y) position in the
                environment
    :param holding: an integer indicating what object the agent is currently
                    holding if any.
                    This parameter only has a non-empty value AFTER the agent
                    performs a 'pick up ...' option and BEFORE it performs a
                    'deliver ...' option.
    :param option_mask: a 1D array indicating the valid options to select next
                        (should not be used when computing rewards, this is only
                        used in some downstream algorithms)
    '''
    # Y pos in bound
    assert pos[1] >= 0 and pos[1] < len(obs)
    # X pos in bound
    assert pos[0] >= 0 and pos[0] < len(obs[0])
    # holding should be a value in the Holding_Obj Enum
    assert holding < len(rw4t_utils.Holding_Obj)
    self.obs = obs
    self.pos = pos
    self.holding = holding
    self.option_mask = option_mask

  def state_to_dict(self):
    return {
        'map': np.array(self.obs, dtype=np.int32),
        'pos': np.array(self.pos, dtype=np.int32),
        'holding': self.holding,
        'option_mask': self.option_mask
    }


class RW4TEnv(gym.Env):

  def get_state(self):
    state = RW4T_GameState(self.map, self.agent_pos, self.agent_holding,
                           self.option_mask)
    state_dict = state.state_to_dict()
    return state_dict

\end{minted}

Write a reward function for the following task: \\

\textit{The LLM receives one of the prompts below, chosen according to the reward function we want it to design.
The prompt for each setting is identical except for (1) descriptions of relevant task spaces (e.g., the flat-reward prompt omits the options space description) and (2) the behavioral specification (``user preference'') string.}

\smallskip
\subsection*{High-Level}
\textbf{Task description:} \\
The task objective is to deliver all objects on the map.
In the task reward, the agent gets a reward of +30 when it successfully delivers an object,
and a step cost of -1 for each time step taken.
The reward function you write does not need to encode the task objective. 

\textbf{Relevant task spaces:} \\
The agent's option/subtask (referred to as HL\_Action in the code) space consists of going to and delivering the two types of objects.
Each option takes multiple action steps to complete. 
Taking a 'go to' option means that the agent will navigate to a supply and pick it up.
Taking a 'deliver' option means that the agent will navigate to the delivery location and drop the object.
Note that the agent has to first go to the object to pick it up before delivering the object.

\textbf{User preference:} \\
The agent should pick up an object type that's the same as the previously delivered object type, 
if there are still objects of that type remaining in the environment.
Otherwise, the agent should pick up an object of a different type.

\textbf{Additional info:} \\
You need to write a reward function to encode this user preference.
The preference function you write will be used together with the task reward to train the agent.
Please make sure NOT to make the reward function stateful (i.e. you should not use function attributes or global variables.)
You should also not write any other helper functions.

\smallskip
\subsection*{Low-Level}
\textbf{Task description:} \\
The task objective is to deliver all objects on the map.
In the task reward, the agent gets a reward of +30 when it successfully delivers an object,
and a step cost of -1 for each time step taken.
The reward function you write does not need to encode the task objective. 

\textbf{Relevant task spaces:} \\
The agent's option/subtask (referred to as HL\_Action in the code) space consists of going to and delivering the two types of objects.
Each option takes multiple action steps to complete. 
Taking a 'go to' option means that the agent will navigate to a supply and pick it up.
Taking a 'deliver' option means that the agent will navigate to the delivery location and drop the object.
Note that the agent has to first go to the object to pick it up before delivering the object.
The agent's action (referred to as LL\_Action in the code) space consists of moving in the four cardinal directions, as well as atomic actions pick and drop.
The agent can only perform LL\_Action "pick" if it is at the same location as the object.

\textbf{User preference:} \\
The agent should avoid yellow danger zones when it is delivering an object.
However, the agent does not need to avoid danger zones when it is going to an object.

\textbf{Additional info:} \\
You need to write a reward function to encode this user preference.
The preference function you write will be used together with the task reward to train the agent.
Please make sure NOT to make the reward function stateful (i.e. you should not use function attributes or global variables.)
You should also not write any other helper functions.

\smallskip
\subsection*{Flat}
\textbf{Task description:} \\
The task objective is to deliver all objects on the map.
In the task reward, the agent gets a reward of +30 when it successfully delivers an object,
and a step cost of -1 for each time step taken.
The reward function you write does not need to encode the task objective. 

\textbf{Relevant task spaces:} \\
The agent's action (referred to as LL\_Action in the code) space consists of moving in the four cardinal directions, as well as atomic actions pick and drop.
The agent can only perform LL\_Action "pick" if it is at the same location as the object.

\textbf{User preference:} \\
The agent should pick up an object type that's the same as the previously delivered object type, 
if there are still objects of that type remaining in the environment.
Otherwise, the agent should pick up an object of a different type.
In addition, the agent should avoid yellow danger zones when it is delivering an object.
However, the agent does not need to avoid danger zones when it is going to an object.

\textbf{Additional info:} \\
You need to write a reward function to encode this user preference.
The preference function you write will be used together with the task reward to train the agent.
Please make sure NOT to make the reward function stateful (i.e. you should not use function attributes or global variables.)
You should also not write any other helper functions.
\end{tcolorbox}

\begin{tcolorbox}[title={Prompt 3: User Prompt for iTHOR},breakable, label={box:prompt_ithor}]
The Python environment is

\begin{minted}[fontsize=\small,baselinestretch=1]{python}
'''
Background:

1) utils:
PnP_LL_Actions = [
    "MoveAhead",
    "RotateLeft",
    "RotateRight",
    "PickupNearestTarget",
    "PutHeldOnReceptacle",
]

class PnP_HL_Actions(Enum):
    pick_apple = 0
    pick_egg = 1
    drop_apple = 2
    drop_egg = 3

class PnP_HL_Actions_With_Dummy(Enum):
    pick_apple = 0
    pick_egg = 1
    drop_apple = 2
    drop_egg = 3
    dummy = 4
'''

import random
import gymnasium as gym
import numpy as np
from gymnasium import spaces
from ai2thor.controller import Controller
from ai2thor.platform import CloudRendering
from typing import Dict, List, Optional

from HierRL.envs.ai2thor.pnp_training_utils import (PnP_HL_Actions,
                                                    PnP_HL_Actions_With_Dummy,
                                                    PnP_LL_Actions)
from HierRL.envs.ai2thor.pnp_config import avoid_stool


class ThorPickPlaceEnv(gym.Env):
  """
  Pick-and-place environment on top of AI2-THOR, using the Gymnasium API.

  Episode structure:
    - reset() loads a kitchen scene, curates it (move/disable/spawn a few
      items), and returns an observation.
    - step(a) applies either low-level nav (Move/Rotate/Look) or a simple HL
      manipulation (Pickup nearest target / Put on nearest receptacle / Drop).
    - reward is currently 0/1 placeholder (see _compute_reward_and_done).
  """
  metadata = {"render_modes": ["rgb_array"]}

  def __init__(
      self,
      scene: str = "FloorPlan20",  # scene id
      pref_dict: Dict[str, List[int]] = avoid_stool,  # preference dictionary
      visibilityDistance: float = 1,  # meters for "visible" flag (not reach)
      grid_size: float = 0.25,  # movement step in meters
      snap_to_grid: bool = True,  # keep motion aligned to grid
      rotate_step_degrees: int = 90,  # degree per rotate action
      render_depth: bool = False,
      render_instance_masks: bool = False,
      target_types=('Apple',
                    'Egg'),  # categories of objects that the agent can pick
      receptacle_types=("SinkBasin", ),  # categories we allow "PutObject" on
      max_steps: int = None,
      low_level: bool = False,  # whether we are working with low-level only
      hl_pref_r=None,
      option: PnP_HL_Actions = None,
      seed: Optional[int] = None,
      render: bool = True):
    super().__init__()
    # Save config
    self.scene = scene
    self.max_steps = max_steps
    self.target_types = set(target_types)
    self.receptacle_types = set(receptacle_types)
    self._rng = random.Random(seed)

    h, w = 600, 600
    platform = None if render else CloudRendering
    self.need_render = render
    self.controller = Controller(
        width=w,
        height=h,
        scene=self.scene,
        gridSize=grid_size,
        snapToGrid=snap_to_grid,
        rotateStepDegrees=rotate_step_degrees,
        renderDepthImage=render_depth,
        renderInstanceSegmentation=render_instance_masks,
        visibilityDistance=visibilityDistance,
        platform=platform)
    self.controller.step(action="Initialize", gridSize=grid_size)

    # Observation: dictionary-based state space.
    self.observation_space = spaces.Dict({
        "apple_1_pos":
        spaces.Box(-3.0, 3.0, (2, ), dtype=np.float32),
        "apple_2_pos":
        spaces.Box(-3.0, 3.0, (2, ), dtype=np.float32),
        "egg_1_pos":
        spaces.Box(-3.0, 3.0, (2, ), dtype=np.float32),
        "egg_2_pos":
        spaces.Box(-3.0, 3.0, (2, ), dtype=np.float32),
        "stool_pos":
        spaces.Box(-3.0, 3.0, (2, ), dtype=np.float32),
        "sink_pos":
        spaces.Box(-3.0, 3.0, (2, ), dtype=np.float32),
        "agent_pos":
        spaces.Box(-3.0, 3.0, (2, ), dtype=np.float32),  # x and z pos
        "agent_rot":
        spaces.Box(0.0, 1.0, (4, ),
                   dtype=np.float32),  # y rot (one-hot encoded)
        "apple_1_state":
        spaces.Discrete(3),  # 0 = on table, 1 = held, 2 = in sink
        "apple_2_state":
        spaces.Discrete(3),  # 0 = on table, 1 = held, 2 = in sink
        "egg_1_state":
        spaces.Discrete(3),  # 0 = on table, 1 = held, 2 = in sink
        "egg_2_state":
        spaces.Discrete(3),  # 0 = on table, 1 = held, 2 = in sink
    })

    # Whether we are working with the low-level only
    self.low_level = low_level

    # Adjust task/subtask horizons
    if max_steps is not None:
      self.max_steps = max_steps
    else:
      if self.low_level:
        self.max_steps = 100
      else:
        self.max_steps = 500

    # Define action spaces
    self.pnp_ll_actions = PnP_LL_Actions
    self.pnp_hl_actions = PnP_HL_Actions
    self.pnp_hl_actions_with_dummy = PnP_HL_Actions_With_Dummy

    # Low level action space: iThor environment commands
    self.ll_action_space = spaces.Discrete(len(self.pnp_ll_actions))
    self.hl_action_space = spaces.Discrete(len(self.pnp_hl_actions))

    # High level action space: Options (pick up/drop specific items)
    # Option values
    self.option = option

    # Initialize environment
    self._setup_env()
    if self.low_level:
      if self.option is None:
        self.option = random.choice(list(self.pnp_hl_actions)).value
      self.action_space = self.ll_action_space
      self.reset(options={'option': self.option})
    else:
      # Replace option with dummy value for high level training
      self.option = self.pnp_hl_actions_with_dummy.dummy.value
      self.action_space = self.hl_action_space
      self.reset()

    self.steps = 0

    # Set preferences
    self.pref_dict = pref_dict

    # Rewards initialization
    self.hl_pref_r = hl_pref_r

    self._per_step_reward = -0.01
    self._obj_drop_reward = 10.0
    self._obj_pick_reward = 10.0
    self._wrong_obj_pick_reward = -5.0
    self._dist_shaping_factor = -0.05
    self._ll_penalty = -1
    self._ll_radius = 1.5
    self._hl_diversity_reward = 5.0

    self.prev_option = self.pnp_hl_actions_with_dummy.dummy.value
    self.c_task_reward = 0
    self.c_pseudo_reward = 0
    self.c_gt_hl_pref = 0
    self.c_gt_ll_pref = 0

    # Used for determining successful placement into receptacle
    self._drop_success = False
    self._pick_apple_success = False
    self._pick_egg_success = False
    
\end{minted}

Write a reward function for the following task: \\

\textit{The LLM receives one of the prompts below, chosen according to the reward function we want it to design.
The prompt for each setting is identical except for (1) descriptions of relevant task spaces (e.g., the flat-reward prompt omits the options space description) and (2) the behavioral specification (``user preference'') string.}

\smallskip
\subsection*{High-Level}
\textbf{Task description:} \\
The task objective is to pick up all apples and eggs on the dining table and place them in the sink.
In the task reward, the agent gets a reward of +10 after it successfully picks up an object and places it in the sink, and a step cost of -0.1 for each time step taken.
The reward function you write does not need to encode the task objective.

\textbf{Relevant task spaces:} \\
The agent's option/subtask (referred to as self.pnp\_hl\_actions in the code) space consists of picking up and placing the two types of objects.
Each option takes multiple action steps to complete. 
Taking a 'pick' option means that the agent will navigate to an object and pick it up.
Taking a 'place' option means that the agent will navigate to the delivery location and place the object there.
Note that the agent has to first go to the object to pick it up before placing the object.

\textbf{User preference:} \\
The agent should pick up an object type that's different from the previously placed object type, as long as there are objects of the other type on the table need to be picked.

\textbf{Additional info:} \\
You need to write a reward function to encode this user preference.
The preference function you write will be used together with the task reward to train the agent.
It can take up to 30 steps to reach an object and pick it up, or to reach the sink and drop it off.
Make sure your reward scaling gives the preference for alternating objects much more weight than the negative step rewards, but still lower than the positive task reward.
Please make sure NOT to make the reward function stateful (i.e. you should not use function attributes or global variables).
You should also not write any other helper functions.

\smallskip
\subsection*{Low-Level}
\textbf{Task description:} \\
The task objective is to pick up all apples and eggs on the dining table and place them in the sink.
In the task reward, the agent gets a reward of +10 after it successfully picks up an object and places it in the sink, and a step cost of -0.1 for each time step taken.
The reward function you write does not need to encode the task objective.

\textbf{Relevant task spaces:} \\
The agent's option/subtask (referred to as self.pnp\_hl\_actions in the code) space consists of picking up and placing the two types of objects.
Each option takes multiple action steps to complete. 
Taking a 'pick' option means that the agent will navigate to an object and pick it up.
Taking a 'place' option means that the agent will navigate to the delivery location and place the object there.
Note that the agent has to first go to the object to pick it up before placing the object. The agent's action (referred to as self.pnp\_ll\_actions in the code) space consists of the primitives for: moving forward, rotating left, rotating right, picking up the closest object, and placing a held object in receptacle.
The agent can only perform the low level pick or place primitive only if the agent is close enough to an object or a receptacle.

\textbf{User preference:} \\
The agent should avoid the stool in the environment both when it is on its way to pick up an egg and place an egg down.
More specifically, the agent should be penalized when it is within 1.5 meters of the stool.
However, the agent does not need to avoid the stool when it is on its way to pick up an apple or place an apple down.

\textbf{Additional info:} \\
You need to write a reward function to encode this user preference.
The preference function you write will be used together with the task reward to train the agent.
Please make sure NOT to make the reward function stateful (i.e. you should not use function attributes or global variables).
You should also not write any other helper functions.

\smallskip
\subsection*{Flat}
\textbf{Task description:} \\
The task objective is to pick up all apples and eggs on the dining table and place them in the sink.
In the task reward, the agent gets a reward of +10 after it successfully picks up an object and places it in the sink, and a step cost of -0.1 for each time step taken.
The reward function you write does not need to encode the task objective.

\textbf{Relevant task spaces:} \\
The agent's action (referred to as self.pnp\_ll\_actions in the code) space consists of the primitives for: moving forward, rotating left, rotating right, picking up the closest object, and placing a held object in receptacle.
The agent can only perform the low level pick or place primitive only if the agent is close enough to an object or a receptacle.

\textbf{User preference:} \\
The agent should avoid the stool in the environment both when it is on its way to pick up an egg and place an egg down.
More specifically, the agent should be penalized when it is within 1.5 meters of the stool.
However, the agent does not need to avoid the stool when it is on its way to pick up an apple or place an apple down. In addition, the agent should pick up an object type that's different from the previously placed object type, as long as there are objects of the other type on the table need to be picked.

\textbf{Additional info:} \\
You need to write a reward function to encode this user preference.
The preference function you write will be used together with the task reward to train the agent.
It can take up to 30 steps to reach an object and pick it up, or to reach the sink and drop it off.
Make sure your reward scaling gives the preference for alternating objects much more weight than the negative step rewards, but still lower than the positive task reward.
Please make sure NOT to make the reward function stateful (i.e. you should not use function attributes or global variables.)
You should also not write any other helper functions.
\end{tcolorbox}

\begin{tcolorbox}[title={Prompt 4: User Prompt for Kitchen},breakable, label={box:prompt_oc}]
The Python environment is

\begin{minted}[fontsize=\small,baselinestretch=1]{python}
'''
Background:
1) Ingredients:
class Ingredients(Enum):
  empty = 0
  tomato = 1
  onion = 2
  lettuce = 3

2) Salad types:
class SoupType(Enum):
  no_soup = 0
  alice = 1
  bob = 2
  cathy = 3
  david = 4

3) All available options:
{'Chop Tomato': 0, 'Chop Lettuce': 1, 'Chop Onion': 2, 
'Prepare David Ingredients': 3, 'Plate David Salad': 4}

4) All available actions:
{0: (0, -1),
 1: (0, 1),
 2: (1, 0),
 3: (-1, 0),
 4: (0, 0)}
If the agent is standing next to a counter, performing an action in the 
direction of the counter interacts with the counter.
For example, if the agent is standing under a counter, performing action 0 
(goes up) interacts with the counter above the agent.
'''

import gymnasium as gym


class OvercookedSimple(gym.Env):

  def get_plain_state(self, raw_info):
    '''
    The output of this function will be the input state in the generated reward
    function.

    The state is a dictionary that maps object names to their locations on the
    map.

    If the object 'obj' is at location (x, y), then state['obj'][y, x] == 1.
    Otherwise, state['obj'][y, x] == 0.
    '''
    num_rows = self.world_size[1]
    num_cols = self.world_size[0]
    state_dict = {}

    # Process Grid Squares Map
    GRIDSQUARES = [
        "Floor", "Counter", "Cutboard", "Bin", "Pot", "FreshTomatoTile",
        "FreshOnionTile", "FreshLettuceTile", "PlateTile"
    ]
    gridsquares_map = raw_info['gridsquare']
    for gridsquare_type in GRIDSQUARES:
      grid_map = gridsquares_map[gridsquare_type].T
      assert grid_map.shape == (num_rows, num_cols)
      state_dict[gridsquare_type] = grid_map

    # Process Object Map
    OBJECTS = ['FreshTomato', 'FreshLettuce', 'FreshOnion'] + [
        'ChoppingTomato', 'ChoppingOnion', 'ChoppingLettuce'
    ] + ['ChoppedTomato', 'ChoppedOnion', 'ChoppedLettuce'] + ['Plate']
    objects_map = raw_info['objects']
    for obj_type in OBJECTS:
      obj_map = objects_map[obj_type].T
      assert obj_map.shape == (num_rows, num_cols)
      state_dict[obj_type] = obj_map

    # Process Agent Map
    agent_map = raw_info['agent_map']['agent-1'].T
    assert agent_map.shape == (num_rows, num_cols)
    state_dict['agent'] = agent_map

    return state_dict

\end{minted}

Write a reward function for the following task: \\

\textit{The LLM receives one of the prompts below, chosen according to the reward function we want it to design.
The prompt for each setting is identical except for (1) descriptions of relevant task spaces (e.g., the flat-reward prompt omits the options space description) and (2) the behavioral specification (``user preference'') string.}

\smallskip
\subsection*{High-Level}
\textbf{Task description:} \\
The task objective is to prepare one David's salad with three ingredients: onion, lettuce, and tomatoes. 
To make this salad, the agent needs to:
a. Chop ingredients (onion, lettuce, and tomatoes). 
   Only one ingredient of each type is needed to complete the salad. 
b. Combined chopped ingredients. 
c. Plate the salad.
The task reward already encodes the task objective.
In the task rewrd, the agent receives a reward of +1 when it completes a salad, and a step cost of -0.01 for each time step taken.
The reward function you write does not need to encode the task objective.  

\textbf{Relevant task spaces:} \\
The agent's option (subtask) space consists of macro cooking actions, such as 'Chop Onion' and 'Plate David Salad'.
Each option takes multiple action steps to complete. 

\textbf{User preference:} \\
The agent should chop an onion after it chops a tomato, and the agent should chop a lettuce after it chops an onion. 
If the ingredients are already chopped or a combined salad already exists, the agent should not chop more ingredients.
The mixing order of the ingredients does not matter.

\textbf{Additional info:} \\
You need to write a reward function to encode this user preference.
The preference function you write will be used together with the task reward to train the agent.
Please make sure NOT to make the reward function stateful (i.e. you should not use function attributes or global variables.)
You should also not write any other helper functions.

\smallskip
\subsection*{Flat}
\textbf{Task description:} \\
The task objective is to prepare one David's salad with three ingredients: onion, lettuce, and tomatoes. 
To make this salad, the agent needs to:
a. Chop ingredients (onion, lettuce, and tomatoes). 
   Only one ingredient of each type is needed to complete the salad. 
b. Combined chopped ingredients. 
c. Plate the salad.
The task reward already encodes the task objective.
In the task reward, the agent receives a reward of +1 when it completes a salad, and a step cost of -0.01 for each time step taken.
The reward function you write does not need to encode the task objective.  

\textbf{Relevant task spaces:} \\
The agent's action (low-level action) space consists of moving in the four cardinal directions (up, down, left, right) and idle.
When an agent is standing next to a counter, performing an action in the direction of the counter interacts with the counter.

\textbf{User preference:} \\
The agent should chop an onion after it chops a tomato, and the agent should chop a lettuce after it chops an onion.
If the ingredients are already chopped or a combined salad already exists, the agent should not chop more ingredients.
The mixing order of the ingredients does not matter.

\textbf{Additional info:} \\
You need to write a reward function to encode this user preference.
The preference function you write will be used together with the task reward to train the agent.
Please make sure NOT to make the reward function stateful (i.e. you should not use function attributes or global variables.)
You should also not write any other helper functions.
\end{tcolorbox}

\subsection{Example LLM-Generated Rewards}
\label{app:llm-output-examples}

Here, we provide examples of high-level, low-level, and flat rewards generated using the prompts from the previous section. \\

\subsubsection{Rescue World}
~
\begin{tcolorbox}[title={Example Flat Preference Reward for Rescue World},breakable, label={box:reward_rw_flat}, before skip=\baselineskip]
\begin{minted}[fontsize=\small,baselinestretch=1]{python}
def get_flat_sa_pref_gpt(state: Dict, action: int) -> Tuple[float, Dict[str, float]]:
    '''
    state: the current state of the environment.
    action: the (low-level) action that the agent is about to perform in the current state.
    '''
    current_pos = state['pos']
    current_map = state['map']
    current_holding = state['holding']
    recently_delivered = -1  # To be inferred from context; assuming function is aware of last delivery
    reward = 0.0
    reward_components = {
        'consistency_bonus': 0.0,
        'danger_zone_penalty': 0.0
    }

    # Reward consistency in type picking
    if action == rw4t_utils.RW4T_LL_Actions.pick.value:
        # Get the object type at current position
        object_type = current_map[current_pos[1], current_pos[0]]
        
        # Check if the object type is the same as the recently delivered one
        if object_type == recently_delivered:
            reward += 1.0  # Encourage same type pick
            reward_components['consistency_bonus'] += 1.0

    # Penalize entering yellow zones when holding an object
    if current_holding != rw4t_utils.Holding_Obj.empty.value:
        # Calculate the next position based on the action taken
        next_pos = current_pos.copy()
        if action == rw4t_utils.RW4T_LL_Actions.go_up.value:
            next_pos[1] -= 1
        elif action == rw4t_utils.RW4T_LL_Actions.go_down.value:
            next_pos[1] += 1
        elif action == rw4t_utils.RW4T_LL_Actions.go_left.value:
            next_pos[0] -= 1
        elif action == rw4t_utils.RW4T_LL_Actions.go_right.value:
            next_pos[0] += 1
        
        # Ensure the next position is within bounds
        if (0 <= next_pos[0] < current_map.shape[1]) and (0 <= next_pos[1] < current_map.shape[0]):
            next_area_type = current_map[next_pos[1], next_pos[0]]
            if next_area_type == rw4t_utils.RW4T_State.yellow_zone.value:
                reward -= 1.0  # Penalize for entering a danger zone
                reward_components['danger_zone_penalty'] -= 1.0

    return reward, reward_components

\end{minted}
\end{tcolorbox}

\begin{tcolorbox}[title={Example Low-Level Preference Reward for Rescue World},breakable, label={box:reward_rw_low}]
\begin{minted}[fontsize=\small,baselinestretch=1]{python}
def get_low_level_pref_gpt(state: Dict, option: int, action: int) -> Tuple[float, Dict[str, float]]:
    '''
    state: the current state of the environment.
    option: the option (subtask) selected by the agent in the current state.
    action: the action that the agent is about to perform in the current state.
    '''

    reward = 0.0
    reward_components = {}

    # Extract necessary information from the state
    agent_pos = state['pos']
    map_state = state['map']
    current_cell = map_state[agent_pos[1], agent_pos[0]]

    # Check if the agent is in a danger zone (yellow zone)
    is_in_yellow_zone = current_cell == rw4t_utils.RW4T_State.yellow_zone.value

    # Determine if the current option is a delivery option
    is_delivery_option = option in [
        rw4t_utils.RW4T_HL_Actions_EZ.deliver_circle.value,
        rw4t_utils.RW4T_HL_Actions_EZ.deliver_square.value
    ]

    # Apply a penalty if the agent is delivering and currently in a yellow zone
    if is_delivery_option and is_in_yellow_zone:
        danger_zone_penalty = -5.0  # User defined penalty
        reward += danger_zone_penalty
        reward_components['danger_zone_penalty'] = danger_zone_penalty
    else:
        reward_components['danger_zone_penalty'] = 0.0

    # No additional reward for being outside danger zones
    return reward, reward_components

\end{minted}
\end{tcolorbox}

\begin{tcolorbox}[title={Example High-level Preference Reward for Rescue World},breakable, label={box:reward_rw_high}]
\begin{minted}[fontsize=\small,baselinestretch=1]{python}
def get_high_level_pref_gpt(state: Dict, prev_option: int, option: int) -> Tuple[float, Dict[str, float]]:
    '''
    state: the current state of the environment.
    prev_option: the last option (subtask) executed by the agent to reach the current state.
    option: the option (subtask) the agent is about to perform in the current state.
    '''
    
    reward = 0.0
    reward_components = {}

    # Determine the type of the previous and current option
    prev_pick_type = None
    curr_pick_type = None
    
    if prev_option == rw4t_utils.RW4T_HL_Actions_EZ.deliver_circle.value:
        prev_pick_type = rw4t_utils.RW4T_State.circle.value
    elif prev_option == rw4t_utils.RW4T_HL_Actions_EZ.deliver_square.value:
        prev_pick_type = rw4t_utils.RW4T_State.square.value

    if option == rw4t_utils.RW4T_HL_Actions_EZ.go_to_circle.value:
        curr_pick_type = rw4t_utils.RW4T_State.circle.value
    elif option == rw4t_utils.RW4T_HL_Actions_EZ.go_to_square.value:
        curr_pick_type = rw4t_utils.RW4T_State.square.value

    # Count remaining objects of each type on the map
    circle_count = (state['map'] == rw4t_utils.RW4T_State.circle.value).sum()
    square_count = (state['map'] == rw4t_utils.RW4T_State.square.value).sum()

    # Add preference reward based on the user preference
    if prev_pick_type is not None:
        if curr_pick_type == prev_pick_type:
            if (curr_pick_type == rw4t_utils.RW4T_State.circle.value and circle_count > 0) or \
               (curr_pick_type == rw4t_utils.RW4T_State.square.value and square_count > 0):
                reward += 5.0  # Reward for picking the same type if available
                reward_components['same_type_pick_bonus'] = 5.0
            else:
                reward_components['same_type_pick_bonus'] = 0.0
        else:
            reward_components['same_type_pick_bonus'] = 0.0

    return reward, reward_components

\end{minted}
\end{tcolorbox}
~
\subsubsection{iTHOR}
~
\begin{tcolorbox}[title={Example Flat Preference Reward for iTHOR},breakable, label={box:reward_pnp_flat}]
\begin{minted}[fontsize=\small,baselinestretch=1]{python}
def get_flat_sa_pref_gpt(state: Dict, action: int) -> Tuple[float, Dict[str, float]]:
    """
    Computes the user preference reward based on agent's proximity to the stool
    and preferences for alternating object types during the pickup/drop process.

    state: the current state of the environment.
    action: the current action the agent is performing.
    """
    reward = 0.0
    reward_components = {
        'stool_penalty': 0.0,
        'alternating_bonus': 0.0
    }
    
    # Agent and stool position
    agent_pos = state['agent_pos']
    stool_pos = state['stool_pos']

    # Calculate distance to stool
    distance_to_stool = np.linalg.norm(np.array(agent_pos) - np.array(stool_pos))

    # If the agent is close to the stool while dealing with eggs, apply penalty
    dealing_with_eggs = (state['egg_1_state'] in [0, 1] or state['egg_2_state'] in [0, 1])
    if dealing_with_eggs and distance_to_stool < 1.5:
        reward -= 2.0  # Penalize for being too close to the stool
        reward_components['stool_penalty'] = -2.0

    # Check if the action is a pickup or drop intention
    if action in [PnP_LL_Actions.index("PickupNearestTarget"), PnP_LL_Actions.index("PutHeldOnReceptacle")]:
        # Determine last placed object type to encourage alternating pick-up
        last_picked = 'apple' if (state['apple_1_state'] == 2 or state['apple_2_state'] == 2) else 'egg'
        available_apples = (state['apple_1_state'] == 0 or state['apple_2_state'] == 0)
        available_eggs = (state['egg_1_state'] == 0 or state['egg_2_state'] == 0)

        if last_picked == 'apple' and available_eggs:
            reward += 5.0  # Encourage picking eggs if last placed was apple
            reward_components['alternating_bonus'] = 5.0
        elif last_picked == 'egg' and available_apples:
            reward += 5.0  # Encourage picking apples if last placed was egg
            reward_components['alternating_bonus'] = 5.0

    return reward, reward_components
\end{minted}
\end{tcolorbox}

\begin{tcolorbox}[title={Example Low-Level Preference Reward for iTHOR},breakable, label={box:reward_pnp_low}]
\begin{minted}[fontsize=\small,baselinestretch=1]{python}
def get_low_level_pref_gpt(state: Dict, option: int, action: int) -> Tuple[float, Dict[str, float]]:
    '''
    state: the current state of the environment.
    option: the option (subtask) selected by the agent in the current state.
    action: the action that the agent is about to perform in the current state.
    '''

    # Define rewards and thresholds
    stool_penalty = -2.0  # Penalty for being too close to the stool
    stool_avoidance_radius = 1.5  # Distance within which to penalize for being too close to the stool

    # Initialize preference reward and its components
    reward = 0.0
    reward_components = {"stool_avoidance_penalty": 0.0}

    # Get stool and agent positions
    stool_pos = np.array(state["stool_pos"])
    agent_pos = np.array(state["agent_pos"])

    # Calculate distance between agent and stool
    dist_to_stool = np.linalg.norm(agent_pos - stool_pos)

    # Determine if the current option involves interacting with eggs
    interacting_with_egg = option in {PnP_HL_Actions.pick_egg.value, PnP_HL_Actions.drop_egg.value}

    # Apply penalty if the agent is too close to the stool and dealing with eggs
    if interacting_with_egg and dist_to_stool < stool_avoidance_radius:
        reward += stool_penalty
        reward_components["stool_avoidance_penalty"] = stool_penalty

    return reward, reward_components
\end{minted}
\end{tcolorbox}

\begin{tcolorbox}[title={Example High-level Preference Reward for iTHOR},breakable, label={box:reward_pnp_high}]
\begin{minted}[fontsize=\small,baselinestretch=1]{python}
def get_high_level_pref_gpt(state: Dict, prev_option: int, option: int) -> Tuple[float, Dict[str, float]]:
    # Initialize reward and components
    preference_reward = 0.0
    reward_components = {
        "alternation_bonus": 0.0,
        "not_alternate_penalty": 0.0
    }
    
    # Determine the type of object just placed and the type to pick
    if prev_option == PnP_HL_Actions.drop_apple.value:
        last_placed_type = "apple"
    elif prev_option == PnP_HL_Actions.drop_egg.value:
        last_placed_type = "egg"
    else:
        last_placed_type = None

    if option == PnP_HL_Actions.pick_apple.value:
        current_pick_type = "apple"
    elif option == PnP_HL_Actions.pick_egg.value:
        current_pick_type = "egg"
    else:
        current_pick_type = None

    # Check items state to determine if alternation is possible
    remaining_apples = state["apple_1_state"] == 0 or state["apple_2_state"] == 0
    remaining_eggs = state["egg_1_state"] == 0 or state["egg_2_state"] == 0

    # Encourage alternating picking
    if last_placed_type and current_pick_type:
        if last_placed_type != current_pick_type:
            # Give a bonus for alternating types
            preference_reward += 8.0  # Scaled to be significant but less than the task completion reward
            reward_components["alternation_bonus"] = 8.0
        elif (current_pick_type == "apple" and remaining_eggs) or (current_pick_type == "egg" and remaining_apples):
            # Apply a penalty if not alternating but possible
            preference_reward += -2.0
            reward_components["not_alternate_penalty"] = -2.0

    return preference_reward, reward_components
\end{minted}
\end{tcolorbox}
~
\subsubsection{Kitchen}
~
\begin{tcolorbox}[title={Example Flat Preference Reward for Kitchen},breakable, label={box:reward_kitchen_flat}]
\begin{minted}[fontsize=\small,baselinestretch=1]{python}
def get_flat_sa_pref_gpt(state: Dict, action: int) -> Tuple[float, Dict[str, float]]:
    '''
    state: the current state of the environment.
    action: the (low-level) action that the agent is about to perform in the current state.
    '''
    # Define a map for action directions
    action_effect_map = {
        0: (0, -1),  # move up
        1: (0, 1),   # move down
        2: (1, 0),   # move right
        3: (-1, 0),  # move left
        4: (0, 0)    # idle
    }
    
    reward = 0.0
    reward_components = {
        'tomato_chopped': 0.0,
        'onion_chopped': 0.0,
        'lettuce_chopped': 0.0
    }

    # Check positions of chopped ingredients
    chopped_tomato_exists = state['ChoppedTomato'].any()
    chopped_onion_exists = state['ChoppedOnion'].any()
    chopped_lettuce_exists = state['ChoppedLettuce'].any()

    # Calculate agent's new position based on the action
    agent_pos = state['agent'].argmax()
    agent_y, agent_x = divmod(agent_pos, state['agent'].shape[1])
    dy, dx = action_effect_map[action]
    new_pos = (agent_y + dy, agent_x + dx)
        
    # Check if the agent is chopping at the new position
    is_near_cutboard = state['Cutboard'][new_pos] == 1
    chopping_tomato_active = is_near_cutboard and state['ChoppingTomato'][new_pos] == 1
    chopping_onion_active = is_near_cutboard and state['ChoppingOnion'][new_pos] == 1
    chopping_lettuce_active = is_near_cutboard and state['ChoppingLettuce'][new_pos] == 1

    # Add preference rewards
    if chopping_tomato_active and not chopped_tomato_exists:
        reward_components['tomato_chopped'] += 0.1  # Encourage chopping tomato
    if chopped_tomato_exists and chopping_onion_active and not chopped_onion_exists:
        reward_components['onion_chopped'] += 0.2  # Encourage chopping onion after chopping tomato
    if chopped_onion_exists and chopping_lettuce_active and not chopped_lettuce_exists:
        reward_components['lettuce_chopped'] += 0.3  # Encourage chopping lettuce after chopping onion

    reward = sum(reward_components.values())
    
    return reward, reward_components
\end{minted}
\end{tcolorbox}

\begin{tcolorbox}[title={Example High-level Preference Reward for Kitchen},breakable, label={box:reward_kitchen_high}]
\begin{minted}[fontsize=\small,baselinestretch=1]{python}
def get_high_level_pref_gpt(state: Dict, prev_option: int, option: int) -> Tuple[float, Dict[str, float]]:
    '''
    state: the current state of the environment.
    prev_option: the last option (subtask) executed by the agent to reach the current state.
    option: the option (subtask) the agent is about to perform in the current state.
    '''
    
    reward = 0.0
    reward_components = {
        "onion_after_tomato": 0.0,
        "lettuce_after_onion": 0.0,
        "avoid_extra_chop": 0.0,
    }

    # Define option indices for ease of reference
    CHOP_TOMATO = 0
    CHOP_LETTUCE = 1
    CHOP_ONION = 2

    # Check for chopped states
    tomato_chopped = state['ChoppedTomato'].any()
    onion_chopped = state['ChoppedOnion'].any()
    lettuce_chopped = state['ChoppedLettuce'].any()

    # User preferences
    if prev_option == CHOP_TOMATO and option == CHOP_ONION:
        reward += 0.5
        reward_components["onion_after_tomato"] = 0.5

    if prev_option == CHOP_ONION and option == CHOP_LETTUCE:
        reward += 0.5
        reward_components["lettuce_after_onion"] = 0.5

    # Avoid chopping ingredients again if they are already chopped
    if (option == CHOP_TOMATO and tomato_chopped) or \
       (option == CHOP_ONION and onion_chopped) or \
       (option == CHOP_LETTUCE and lettuce_chopped):
        reward -= 1.0
        reward_components["avoid_extra_chop"] = -1.0

    return reward, reward_components
\end{minted}
\end{tcolorbox}
 \newpage
\section{User Study Details}
\label{app:userstudy}
\subsection{Experiment Protocol}
We conducted a user study to evaluate how well agent policies trained with hierarchical rewards (experimental group) generated by language models are perceived to align with given behavioral specifications compared to those trained with flat rewards (control group), also generated using language models.
Each participant was randomly assigned to either the Rescue World or Kitchen domain. 

\begin{enumerate}
    \item \textbf{Consent and Study Overview.}
    Participants were first presented with a detailed overview of the study, including its purpose, procedures, and any potential risks. 
    An IRB-approved consent form was provided, and participants were required to give informed consent before proceeding.
    \item \textbf{Demographic Questionnaire.}  
    After providing consent, participants completed a brief demographic questionnaire, where we asked for their age and sex.
    \item \textbf{Domain Introduction.}  
    Participants were introduced to the assigned domain through textual descriptions accompanied by screenshots. 
    This step was designed to ensure that they had sufficient context to understand the environment and the tasks performed by the agent.
    \item \textbf{Presentation of Behavioral Specifications and Attention Checks.}  
    Next, participants were shown the behavioral specifications the agent was expected to follow. 
    To ensure they carefully read these specifications, we included attention check questions.
    For example, in the Rescue World domain, where the agent must handle two object types (food and medical kits) and may encounter avoid danger zones (marked by yellow grids), part of the \textit{safety} specification states that the robot should ``avoid yellow danger zones when it is delivering an object''. 
    We asked participants, whether according to the specifications, it would be considered safe for the robot to ``go through danger zones while delivering food''.
    While participants were not required to answer these questions correctly to proceed, we filtered out all responses with incorrect answers during data analysis to ensure data quality.
    \item \textbf{Video Evaluation.} 
    To ensure consistent evaluation, only policies that successfully completed the task were shown.
    Additionally, to control for variability, all videos shown to a participant were drawn from policies trained with the same random seed, although they could originate from different reward candidates.
    In the \textit{Rescue World} domain, each participant viewed 6 videos: 3 showing policies trained with flat rewards $\tilde{r}_{flat}$ and 3 showing policies trained with hierarchical rewards $(\tilde{r}_H, \tilde{r}_L)$. 
    In the \textit{Kitchen} domain, participants viewed 4 videos: 2 per reward method. Fewer videos were shown in this domain because, in some cases, only 2 out of 8 flat reward candidates produced policies capable of successfully completing the task.
    Questions for each video appeared on the same page, and participants were allowed to replay the videos as many times as they wished.
    After each video, participants first answered a multiple-choice question to verify they had watched the video (e.g., ``What order did the robot deliver the objects in?'').  
    They were then asked to rate how well the agent's behavior aligned with the specified behaviors using a 5-point scale, with 1 indicating ``least aligned'' and 5 indicating ``most aligned''. 
    Participants could optionally provide any comments explaining their ratings.
    \item \textbf{Final Feedback and Compensation.}  
    At the end of the study, participants were invited to leave any additional feedback about their experience. 
    On average, participants took 14.2 minutes to complete the study (SD = 8.7 minutes).
    All participants who completed the full study were compensated \$3 for their time.
\end{enumerate}

\subsection{Participants}
We conducted the user study on Prolific, recruiting a total of 40 participants from the United States.
After applying attention check filters, we obtained usable data from 30 participants, with 15 participants assigned to each domain. 
Among these 30 participants, the youngest was 20, the oldest was 67, and the median age was 35.
The sex distribution was fairly balanced, with 14 female participants, 15 male participants, and 1 participant identifying as ``non-binary / third gender.''

\subsection{Data Analysis}
We filtered out data from participants who incorrectly answered the attention check questions about the behavioral specifications.
For participants who passed the attention checks, if they answered a video content question incorrectly, we excluded their ratings for that specific video but retained their responses for other videos.

To generate Fig.~\ref{fig:human-ratings}, we computed the average rating each participant assigned to \textit{Flat} and \textit{Hier} policies, and then aggregated these per-participant means to report group-level comparisons.
Since each participant evaluated both \textit{Flat} and \textit{Hier} policies trained with the same random seed, we used the Wilcoxon Signed-Rank test to assess the statistical significance of the observed rating differences \cite{hoffman2020primer}.
All reported results are based on the filtered dataset, with the final sample sizes specified in the Participants section.
 \newpage
\section{Code Availability and Release}
A repository containing the source code developed for this project is available at \href{https://github.com/unhelkarlab/hrdl}{\url{https://github.com/unhelkarlab/hrdl}}.
\fi




\end{document}